\documentclass[journal]{IEEEtran}
\usepackage{amsmath,amsfonts, amsthm,amssymb}
\usepackage{algorithmic}
\usepackage{algorithm}
\usepackage{graphicx}      % include this line if your document contains figures\usepackage{array}
\usepackage[caption=false,font=normalsize,labelfont=sf,textfont=sf]{subfig}
\usepackage{textcomp}
\usepackage{stfloats}
\usepackage{url}
\usepackage{verbatim}
\usepackage{graphicx}
\usepackage{bbm}
\usepackage{makecell}

\usepackage{cite}

\newtheorem{theorem}{Theorem}
\newtheorem{lemma}{Lemma}

\newcommand{\RNum}[1]{\uppercase\expandafter{\romannumeral #1\relax}}
\usepackage[inline]{enumitem}
\usepackage[dvipsnames]{xcolor}

\usepackage{ifthen}
\newboolean{showmodification}
\setboolean{showmodification}{true}

\newcommand{\Tulga}[1]{\textcolor{Red}{[Tulga: #1]}}
\newcommand{\Siyuan}[1]{\textcolor{Green}{[Siyuan: #1]}}
\newcommand{\Congkai}[1]{\textcolor{Cyan}{[Congkai: #1]}}
\newcommand{\Yufei}[1]{\textcolor{Purple}{[Yufei: #1]}}
\newcommand{\Discussion}[1]{\textcolor{Blue}{[Discussion: #1]}}

\ifthenelse{\boolean{showmodification}}{
}
{
\renewcommand{\Tulga}[1]{}
\renewcommand{\Siyuan}[1]{}
\renewcommand{\Congkai}[1]{}
\renewcommand{\Yufei}[1]{}
\renewcommand{\Discussion}[1]{}
}
\usepackage[dvipsnames]{xcolor}

\newboolean{PngBool}
\setboolean{PngBool}{false} % Set to true or false

\usepackage[normalem]{ulem} % for spacing issue for "subequations"
\usepackage{etoolbox}% for spacing issue for "subequations"

\makeatother

\begin{document}

% \title{Unifying High Performance Driving: Safe Envelope Model Predictive Control}
\title{Spatial Envelope MPC: High Performance Driving without a Reference}
% \title{Reimagining Envelope MPC: A Unifying Planning and Control Approach for Reference-free High Performance Driving }
% \title{Reinventing Envelope MPC: A Unifying Planning and Control Approach for Reference-free High Performance Driving}
\author{Siyuan Yu$^\dagger$, Congkai Shen$^\dagger$, Yufei Xi, James Dallas, Michael Thompson, John Subosits, Hiroshi Yasuda,  Tulga~Ersal$^*$ 
\thanks{$^\dagger$These authors contributed equally to this work.}
\thanks{* Corresponding author}
\thanks{ S. Yu, C. Shen, Y. Xi, and T. Ersal are with the Department of Mechanical Engineering, University of Michigan, Ann Arbor, MI 48109. (email: \{johnysy, cosh, xiyufei, tersal\}@umich.edu)}
\thanks{J. Dallas, M. Thompson, J. Subosits, H. Yasuda are with Toyota Research Institute, Los Altos, CA 94022. (email: \{james.dallas, michael.thompson, hiroshi.yasuda, john.subosits\}@tri.global)}
\thanks{Toyota Research Institute provided funds to support this work.}
}

% \markboth{IEEE Transactions on Robotics,~Vol.~?, No.~?, ?~2024}{}%
% \markboth{IEEE Transactions on Robotics:~Vol.~?, No.~?, ?~2024}{}%


\maketitle

\begin{abstract}
This paper presents a novel envelope-based model predictive control (MPC) framework designed to enable autonomous vehicles to handle high-performance driving across a wide range of scenarios without a predefined reference.
In high-performance autonomous driving, safe operation at the vehicle’s dynamic limits requires a real-time planning and control framework capable of accounting for key vehicle dynamics and environmental constraints when following a predefined reference trajectory is suboptimal or even infeasible.	
State-of-the-art planning and control frameworks, however, are predominantly reference-based, which limits their performance in such situations.
To address this gap, this work first introduces a computationally efficient vehicle dynamics model tailored for optimization-based control and a continuously differentiable mathematical formulation that accurately captures the entire drivable envelope.
This novel model and formulation allow for the direct integration of dynamic feasibility and safety constraints into a unified planning and control framework, thereby removing the necessity for pre-defined references.
The challenge of envelope planning, which refers to maximally approximating the safe drivable area, is tackled by combining reinforcement learning with optimization techniques.
The framework is validated through both simulations and real-world experiments, demonstrating its high performance across a variety of tasks, including racing, emergency collision avoidance and off-road navigation. 
These results highlight the framework’s scalability and broad applicability across a diverse set of scenarios.
\end{abstract}


\section{Introduction}\label{sec:intro}
%  Driving advanced topic
High-performance autonomous driving technology has advanced rapidly and significantly in the past decade \cite{betz2022autonomous}.
The grand opening of the first autonomous racing competition has thrust autonomous driving technology into the spotlight, showcasing performance that rivals expert human drivers in tasks demanding high levels of precision and speed \cite{wischnewski2022indy}.	
Recent developments in autonomous driving have enabled reliable performance in advanced collision avoidance capabilities \cite{zhang2020optimization, funke2016collision, wurts2021collision}, complex drifting maneuvers \cite{weber2023modeling, goh2016simultaneous, goh2024beyond}, and challenging off-road navigation \cite{dallas2021terrain, shen2023efficient, yu2021nonlinear}. 
These advancements have greatly expanded the potential applications of autonomous vehicles across various fields.	
To execute these extreme maneuvers safely, the vehicle must operate at the limits of its capabilities.
However, this is a challenge, because even a small deviation from the desired trajectory can result in catastrophic outcomes.

In this regard, the autonomous system should be designed to effectively utilize all feasible operational regions including the limits, ensuring performance is not unnecessarily sacrificed for safety via overly conservative constraints.
For instance, every inch of the race track is critical for drivers to adjust their strategies and enhance their performance.	
In addition, the algorithm's design should ensure scalability when applied to a diverse set of scenarios with varying levels of complexity. 
In this context, the term `scalability' refers to the algorithm's ability to handle an increasing amount of tasks or scenarios without a significant compromise in performance or without significant redesign.
It is also desirable that the system be capable of generating optimal trajectories online in real-time without having to depend on a predefined reference, because deviations from the original plan may render the predefined reference suboptimal or even infeasible.

However, as the literature review in Sec. \ref{sec:background} reveals, existing methods typically rely on predefined references and therefore either limit vehicle performance to the quality of that reference or do not provide optimal and scalable solutions in performance-demanding scenarios.	

% \subsection{Original Contribution}
To address this gap, this paper presents a novel, spatial envelope model predictive control (MPC) framework for reference-free high performance driving.
The proposed framework builds upon a new, computationally efficient vehicle dynamics model tailored for closed-loop optimization based planning and control, capturing the essential dynamics required for aggressive maneuvers. 
A twice continuously differentiable mathematical formulation of the entire driving envelope is introduced to conservatively estimate the drivable region to be used in MPC.
This enables the MPC to break from the restrictive constant-speed assumption previously used in spatial envelope MPC \cite{beal2012model, brown2017safe, wurts2020collision, wurts2021collision}, and instead optimize speed, as well, to maximize performance while maintaining safety.

To the authors' knowledge, this is the first published MPC algorithm that is experimentally validated for safe and effective high-performance driving at the handling limits in a fully reference-free setting.
The algorithm is validated across a wide range of scenarios, including racing, off-road navigation and emergency collision avoidance, demonstrating both generality and real-world applicability.

Finally, a new spatial envelope planning technique is introduced to further enhance applicability. A hybrid approach that combines optimization-based formulation with reinforcement learning is developed to segment the drivable area into blocks, enabling scalable planning in complex environments.

The original contributions are summarized as follows:

\begin{enumerate}
    \item A validated 3-DoF single-track dynamic model that accounts for longitudinal load transfer and the friction circle limit, while remaining computationally efficient in a fully reference-free setting.
    \item A hard constraint formulation to mathematically express the spatial envelope with guaranteed conservativeness.
    \item A real-time Model Predictive Control (MPC) formulation that leverages the first two contributions to optimize vehicle trajectories online without any predefined path.
    \item A reinforcement learning approach to design a set of blocks to approximate arbitrary shapes of spatial envelopes in real time.
    \item Validation of the proposed MPC formulation in racing, emergency collision avoidance and off-road environments.

\end{enumerate}

% \subsection{Organization}
The rest of the paper is organized as follows. Sec. \ref{sec:background} reviews the relevant literature. Sec.~\ref{sec:methods}-A describes the 3 DoF single-track vehicle dynamics. Sec.~\ref{sec:methods}-B describes the MPC formulation including the conservative spatial envelope constraints. Sec.~\ref{sec:methods}-C describes the real-time spatial envelope planner. Sec.~\ref{sec:model_fidelity} describes the model fidelity test. The results and discussion of the proposed spatial envelope MPC are presented in Sec.~\ref{sec:model_fidelity}.
From Sec. \ref{sec:Results}-A to E, the simulation and experimental results of spatial envelope MPC are conducted and analyzed in multiple scenarios.
In Sec.~\ref{sec:methods}-F, the proposed spatial envelope planning technique is presented and discussed. 
Finally, Sec.~\ref{sec:conclusion} concludes the study.

\section{Background}
\label{sec:background}
A conventional autonomy stack typically divides the navigation task into two layers.
The first layer focuses on the generation of a collision-free path or trajectory as a reference, whereas the second layer is designed to track this reference \cite{stano2023model}.

%  Sampling, graph search based algorithm
One common approach for the first layer is graph search.
In \cite{erke2020improved}, an improved A* algorithm is used to plan a smooth path for the vehicle in urban areas. 
To tackle dynamic environments, D* \cite{osmankovic2017all} is used to reuse previous information. 
To account for nonholonomic constraints, the hybrid-state A* \cite{dolgov2008practical} uses a kinematic model to expand the graph and the Reeds-Sheeps curve as the heuristic \cite{reeds1990optimal}. 
However, the size of the graph increases exponentially with higher resolutions, hindering scalability.
In \cite{lavalle1998rapidly}, the Rapidly Exploring Randomized Tree (RRT), a type of incremental sampling method, gradually constructs a tree graph towards the goal point based on heuristics and randomized sampling. 
In \cite{karaman2011sampling}, RRT* introduces a rewiring function to achieve asymptotic optimality. 
Due to the flexibility of the steering function, several extensions of the RRT-based algorithm have been developed, such as incorporating a kinematic vehicle model in \cite{webb2013kinodynamic} and considering dynamic constraints in \cite{hwan2013optimal}. 
Although RRT-based methods have been widely adopted, their inherent sampling nature creates a trade-off between computation time and solution quality, and the solution is not repeatable.
Thus, despite their success and widespread use in motion planning, graph search methods suffer from issues such as lack of considerations in dynamic constraints, the curse of dimensionality, and nonrepeatability.	

% Vehicle dynamics
To fully exploit the mobility of the vehicle, a dynamic model can be leveraged to understand its limits.
In \cite{beal2012model}, a linear single-track model is utilized to constrain the vehicle's motion within the stable region, assuming constant longitudinal speed and small slip angles.	
In \cite{turri2013linear}, a linear model incorporating load transfer is developed to address the lane-keeping problem using curvilinear coordinates along the lane.	
In \cite{alcala2020autonomous}, a similar model is employed in a racing application, but it is nonlinear because the longitudinal speed is not assumed to be constant.	
Curvilinear coordinates allow for the natural expression of race progression and track boundaries.	
However, the model's accuracy is compromised when the vehicle significantly deviates from the reference line. 
In \cite{goh2020toward}, the friction circle limit is considered by bounding the maximum lateral tire force within the Fiala brush tire model despite the fact that this approach can introduce numerical issues during the optimization process.	
Thus, existing vehicle models either oversimplify the dynamics or lack computational suitability for a reference-free framework.

% MPC with obstacle avoidance
Model predictive control (MPC) effectively harnesses the benefits of having a dynamical model to optimize the vehicle’s trajectory and integrate this knowledge with the desired costs and constraints.	
Its capability to consider obstacle constraints naturally delineates the drivable areas on the map so that it can plan safe and optimal trajectories.	
In \cite{liu2017combined, hu2020steering, lim2008nonlinear}, MPC is integrated with obstacle avoidance constraints to ensure safe vehicle operation.	
To precisely formulate the hazard zones, researchers utilized exact polygons to define collision avoidance constraints, thereby maximizing the traversable area \cite{fan2024exact}.	
Expanding on static obstacle avoidance scenarios, moving obstacle avoidance has also been addressed \cite{li2019dynamic, febbo2017moving}.
Beyond typical on-road driving scenarios, MPC has also been employed to accommodate a terramechanics-based vehicle model for navigating off-road environments without predefined trajectories \cite{dallas2021terrain}.	
Despite the success MPC has achieved in this field, these solutions have only been validated in relatively simple scenarios and do not scale well with the increasing complexity that vehicles might encounter in real-world environments.	

% MPC Tracking
To avoid these challenges, a conventional autonomy stack may choose to employ MPC only as a tracker in the second layer.
In studies such as \cite{turri2013linear,gutjahr2016lateral}, linear MPC is used to track the vehicle along a given trajectory.	
To further enhance the mobility of the vehicle model, nonlinear MPC, which incorporates nonlinear vehicle dynamics, is introduced to handle vehicle control \cite{beal2012model,xu2019design, hu2020lane}.	
The maturity of this technique has contributed to the delivery of state-of-the-art, high-performance autonomous driving solutions.	
In \cite{brown2019coordinating}, tire forces between longitudinal and lateral directions are coordinated using nonlinear MPC with a Frenet coordinate-based vehicle model to track a given trajectory.	
Building on this, researchers adopted a hierarchical approach in which the predictive control layer integrates a single-track model with longitudinal weight transfer dynamics, while the chassis control layer manages lateral brake distribution to track a pre-defined racing trajectory.
\cite{dallas2023adaptive}.	
This divide-and-conquer strategy has demonstrated significant advantages in scalability by addressing collision concerns in the planning layer and limiting the controller to trajectory following.
However,  it also raises new issues. 
Namely, the framework's performance depends on the quality of the path produced by the high-level planner.
Thus, planning optimal and dynamically feasible trajectories in real-time remains a challenge.

% Envelope MPC
To retain the scalability benefits offered by the trajectory tracking method without overly restricting the vehicle's operating region, researchers developed a spatial envelope based MPC method that explicitly defines the drivable area as spatial constraint for the vehicle \cite{brown2017safe,lee2015automated,jiang2021novel}.
Rather than informing the MPC of hazard zones, this method constrains the MPC to plan trajectories exclusively within the predefined safe area.	
The advantage of prescribing a safe area lies in its ability to effectively reduce the number of constraints needed to enforce the safety requirements.	
Researchers in \cite{yu2022autonomous, wurts2020collision} employ this concept, adopting both linear and nonlinear MPC for emergency collision maneuvers, whereas the work in \cite{wurts2021collision} extends the envelope-based approach to curved roads for broader applicability.	
However, the envelope-based MPC discretizes the safe driving area into segments, constraining trajectory points to remain within a single segment.	
This approach requires identifying the relevant segment for each specific trajectory point prior to solving the MPC, which presents a problem, because segment identification is neither a continuous nor differentiable process.	
Consequently, these methods often assume a fixed speed, further restricting the vehicle's operating space by disallowing longitudinal maneuvers. 
This limitation significantly hinders the practical application of this concept.

% Tube planning
% Research on drivable envelopes for vehicles generally falls into two main categories: defining a safe region for vehicle maneuvers based on sensory data \cite{bernhard_risk-based_2022}, and applying these envelopes to control algorithms \cite{andreasson_autonomous_2015,pecora_mission-dependent_2012,wurts2021collision}. However, the challenge of planning drivable envelopes from a given safety region remains unexplored, despite its critical importance for the practical deployment of envelope-based control algorithms.
Another problem to address with spatial envelope MPC is the design of the envelopes.
In \cite{bernhard_risk-based_2022}, researchers defined a drivable envelope as a set of polygonal sections within which every trajectory remains inside the safety region. 
These envelopes were defined heuristically based on a pre-planned trajectory. 
However, this method does not encompass the entire drivable region along the route.
In \cite{wurts2021collision}, researchers sampled parallelograms from road boundaries to form the drivable envelopes. 
Nevertheless, the manual determination of the sampling interval restricts the applicability of this approach to diverse road conditions.
For optimizing polygonal shapes within a given boundary, similar efforts have been made in the field of geometric optimization. 
For example, algorithms have been developed to find the largest area rectangle within a given contour, and methods for solving the rectangle packing problem have been proposed \cite{daniels_finding_1997,molano_finding_2012,liu_optimization_2014,korf_optimal_2010}. 
However, these algorithms either cannot ensure that the resulting rectangles intersect or cannot provide solutions for rectangles of varying sizes, limitations that are problematic for planning polygons for drivable envelopes.

% Machine Learning algorithm
In contrast to the conventional autonomy stack architecture, the end-to-end approach has become another trend due to rapid developments in machine learning. 
In \cite{bojarski2017explaining}, the PilotNet CNN architecture is proposed for autonomous driving. 
The PilotNet directly models the relationship between  input images and steering values.
However, the CNN architecture cannot capture temporal information. 
In \cite{weiss2020deepracing}, a deep neural network (DNN) called AdmiralNet, a convolutional neural network (CNN) integrated with a Long Short-Term Memory (LSTM) network, is used to handle time-varying information and is trained on annotated data generated by human operators. 
However, employing humans to generate training data is expensive, especially for a DNN, which requires a large amount of data. 
Instead of treating DNN as a regression tool, additional research efforts have focused on the field of reinforcement learning. 
In \cite{niu2020two} and \cite{remonda2021formula}, Deep Deterministic Policy Gradient (DDPG), one of the most representative actor-critic (AC) algorithms, is used to learn the policies. 
Although DDPG is effective for continuous control, it suffers from the overestimation of Q values. 
To address this issue, Soft Actor-Critic (SAC) is proposed in \cite{song2021autonomous}, which adopts a maximum entropy framework to encourage exploration and improve robustness.
However, for the end-to-end approach, interpreting the learned policy is still an open question. 
Additionally,  performance is not guaranteed when out-of-distribution events occur due to its data-driven nature.

Based on the literature review above, a key gap has been identified: the absence of a real-time reference-free trajectory planning and control framework that can push the vehicle to its limits while remaining scalable to a diverse set of scenarios.	
Bridging this gap is essential for advancing high-performance autonomous driving and is the goal of this work.

\section{Methods}\label{sec:methods}

\subsection{Vehicle Dynamics} \label{sec:VD}

% Although previous work has demonstrated that the 3 Degree-of-Freedom (DoF) model can achieve a balance between model fidelity and computational performance \Tulga{add citation}, it still encounters a significant dilemma.	
% The model can either be overly simplified, failing to capture the necessary dynamics for accurate predictions in dynamic maneuvers, or overly complex, making it computationally prohibitive for real-time applications.	
In this section, a single-track vehicle model is introduced that effectively captures the essential dynamics in highly dynamic maneuvers. 
The model is also specifically tailored for integration with optimization-based motion controllers, significantly enhancing its real-time capability. 
The model is illustrated in Fig.~\ref{fig:KS}-A.

The state and control vectors are defined as: 
\begin{equation}
  \xi:=\left[\begin{array}{c}
x \\
y \\
v \\
r \\
\psi \\
u_x \\
\delta_f \\
a_x \\
\end{array}\right]=\left[\begin{array}{c}
\text { global x position of C.G. point} \\
\text { global y position of C.G. point} \\
\text { lateral speed } \\
\text { yaw rate } \\
\text { yaw angle } \\
\text { longitudinal speed } \\
\text { front tire steering angle } \\
\text { longitudinal acceleration }\\
\end{array}\right]
\label{eq:state_vector}
\end{equation}
\begin{equation}
  \zeta:=\left[\begin{array}{c}
  \dot{\delta}_f \\
  j_x\\
\end{array}\right]=\left[\begin{array}{c}
\text { front tire steering rate} \\
\text { longitudinal jerk} \\
\end{array}\right]
\label{eq:control_vector}
\end{equation}
The vehicle dynamics are represented in state space as: 
\begin{equation} \label{eq:state_space}
  \dot{\xi} = \mathcal{V}[\xi(t),\zeta(t)] = A(\xi) +B\zeta
\end{equation}
\begin{gather}
 \label{eq:VM} A(\xi)=\left[\begin{array}{c}
u_x\cos \psi - v \sin \psi\\
u_x\sin \psi + v \cos \psi\\
(F_{yf} \cos\delta_f + F_{xf} \sin \delta_f+F_{yr})/M - u_x r\\
((F_{yf} \cos\delta_f + F_{xf} \sin\delta_f )L_f-F_{yr}L_r)/I_{zz} \\
r\\
a_x + r  v - \frac{F_{yf} \sin \delta_f}{M}\\
0 \\
0 \\
\end{array}\right] \\
\label{eq:VB}B^{T} =\left[\begin{array}{llllllll}
 0 & 0 & 0 & 0 & 0 & 0 & 1 & 0\\
 0 & 0 & 0 & 0 & 0 & 0 & 0 & 1
\end{array}\right]
\end{gather}

Here $M$ is the vehicle total mass, $I_{zz}$ is the moment of inertia, and $L_f$ and $L_r$ are distances from the vehicle center of gravity to the vehicle front and rear axles, respectively. 
$F_{yf}$ and $F_{yr}$ are the lateral forces of the front and rear tires. 
$F_{xf}$ is the longitudinal force of the front tire.
Note that $a_x$ is defined as $\frac{Fx}{M}$ instead of $\frac{du_x}{dt}$, where $F_x$ is the total tire force along the longitudinal direction. 
By this definition, longitudinal load transfer is expressed in a linear form, which speeds up the computation of lateral tire forces and hard constraints for the friction circle.
The vehicle is assumed to be rear-wheel driven, with braking applied on both front and rear wheels with a specific braking ratio $b_r$.
Thus, the front-to-rear distribution of longitudinal tire forces is expressed as follows:
\begin{align}
    \label{eq:acc_longitudinal}F_{xf} &= \begin{cases}
        0 &\text{if }F_x \geq 0\\
        b_r F_x &\text{if }F_x < 0
    \end{cases}\\
    \label{eq:br_longitudinal}F_{xr} &= \begin{cases}
        F_x &\text{if }F_x \geq 0\\
        (1 - b_r) F_x &\text{if }F_x < 0
    \end{cases}
\end{align}
where $F_{xf}$ and $F_{xr}$ are the front and rear longitudinal tire forces, respectively. 
However, these expressions are not suitable for optimal control due to the if-condition not being differentiable.
Therefore, a proximal representation of the if-condition is employed using the sigmoid function:
\begin{align}
    \mathbbm{1} &= (1 - \frac{1}{1 + e^{-p_{f} F_x}}) \\
    F_{xf} &= \mathbbm{1} \cdot b_r F_x \\
    F_{xr} &= F_x - F_{xf}
\end{align}

where \( p_{f} \) is a tunable parameter that controls the sharpness of the transition between the \textbf{if} and \textbf{else} conditions. It is also used in Eq.~\eqref{eq:Lateral_Force}.
% Here $\mathbbm{1}$ is an indicator function; i.e., if $F_x > 0$, then $\mathbbm{1} \approx 0$. 
% Otherwise, $\mathbbm{1} \approx 1$. 
% Let $a_{xf} = \frac{F_{xf}}{M}$ and $a_{xr} = \frac{F_{xr}}{M}$.
% \Tulga{Inconsistent notation: Here we type $a_{x_f}$, but below we type $a_{xf}$}\Congkai{Changed}
% Then, these two terms are expressed as below:
% \begin{equation}\label{eq:acc}
%     \begin{cases}
%         a_{xf} &= \mathbbm{1}  a_x  b_{r}\\
%         a_{xr} &= a_x - a_{xf}
%     \end{cases}
% \end{equation}
% \Tulga{Why do we define $a_{x_f}$ and $a_{x_r}$? It seems like they do not appear anywhere else in the paper.}\Congkai{$a_{xf}$ and $a_{xr}$ are not used. I prefer to delete them.}
The longitudinal load transfer is expressed as follows:
\begin{align}
    \label{eq:front_load}F_{zf} &= \frac{L_r}{L_f + L_r} Mg  - K_{z} a_x \\
    \label{eq:rear_load}F_{zr} &= \frac{L_f}{L_f + L_r} Mg + K_{z} a_x \\
    \label{eq:load_const}K_{z} &= \frac{M g h}{L_f + L_r}
\end{align}
where $g$ is the gravitational constant.

By derating the lateral tire force with the longitudinal tire force as discussed in \cite{laurense2019integrated}, the formulation is structured such that the longitudinal tire force required by $a_x$ is always satisfied first and the maximum lateral force is then derived considering that the combined total tire force cannot exceed the maximum allowable friction force (see Fig.~\ref{fig:KS}-B1). 
The maximum lateral tire force could thus be expressed as: 
\begin{align} \label{eq:fymax}
    F_{y\bullet_{\max}} &= \sqrt{ 1 - \left(\frac{F_{x\bullet}}{\mu_{\bullet} F_{z\bullet}} \right)^2 } \mu_{\bullet} F_{z\bullet}
\end{align}
where $\bullet$ is a placeholder for $f, r$, representing the front and rear axles, respectively, and $\mu$ is the friction coefficient.
However, during the initialization of the optimization process, it is possible that the magnitude of $\mu_{\bullet} F_{z\bullet}$ becomes smaller than $F_{x\bullet}$, leading to numerical errors in the solver. 
To avoid this issue, a softplus function is applied prior to taking the square root:
\begin{align}\label{eq:Lateral_Force}
   F_{y\bullet_{\max}} &= \sqrt{\frac{\log (1 + e^{p_f (1 - (\frac{F_{x\bullet}}{\mu_\bullet F_{z\bullet}})^2) })}{p_f}} \mu_\bullet F_{z\bullet}
\end{align}
When $|\mu_{\bullet} F_{z\bullet}| > |F_{x\bullet}|$, Eq. \eqref{eq:Lateral_Force} \ is approximately identical to Eq.~\eqref{eq:fymax}. When $|\mu_{\bullet} F_{z\bullet}| < |F_{x\bullet}|$, $F_{y\bullet_{\max}} \approx 0$.

A sigmoid function is used to relate the lateral tire force $F_{y\bullet}$ and the slip angle $\alpha_\bullet$ as depicted in Fig.~\ref{fig:KS}-B2
The analytical expression is shown below:
\begin{align}
    F_{y\bullet} &= -2 F_{y\bullet_{\max}} \left(  \left(1 + e^{\frac{- 2 C_{a\bullet} \alpha_\bullet }{F_{y\bullet_{\max}}} }\right)^{-1}   -0.5 \right)
\end{align}
where $C_{a\bullet}$ denotes the tire cornering stiffness. 
The sigmoid function captures the key considerations of the Fiala tire model used in \cite{dallas2023adaptive}. 
Namely, the slope around the linear region ($\alpha_\bullet \ll 1$ ) is $C_{a\bullet}$ and the magnitude is bounded by $F_{y\bullet_{\max}}$. 
With this formulation, the vehicle dynamics achieves $C^2$ continuity, which the optimizer favors, thus significantly improving computational speed in optimization.

\subsection{Spatial Envelope Model Predictive Control (Spatial Envelope MPC)}\label{sec:env_mpc}
The nonlinear Model Predictive Control (MPC) framework is formulated such that the controller is capable of optimizing and tracking the trajectories in real time without any predefined trajectories. 
It directly outputs the control commands to the vehicle. 

The optimal control problem (OCP) solved within the MPC over a receding horizon is given as follows. 
\begin{align}
\label{eq:cost} \mathop{\mathrm{minimize}}\limits_{\xi,\ \zeta}\quad \quad &J =\int_{t_0}^{t_f}\mathcal{I}[\xi(t),\zeta(t)]\, dt + \mathcal{G}[\xi(t_f)]\\
\label{eq:Dynamic}\text{subject to} \quad \quad &\dot{\xi}(t) = \mathcal{V}[\xi(t),\zeta(t)]\\
\label{eq:state}& \xi_\text{min}\leq \xi(t) \leq \xi_\text{max}\\
\label{eq:ctr}& \zeta_\text{min}\leq \zeta(t) \leq \zeta_\text{max}\\
\label{eq:tire_force_cons}& \mathbf{\mathcal{T_F}}[\xi(t)]\leq \mathbf{0} \\
\label{eq:envelope_constrraint} & \mathbf{\mathcal{E}}[\xi(t)] \leq \mathbf{0}
\end{align}
$J$ represents the cost function, $\mathcal{I}$ denotes the stage cost for each point along the predicted trajectory, and $\mathcal{G}$ signifies the cost-to-go term at the final trajectory point.
The Optimal Control Problem (OCP) is constrained to adhere to the vehicle dynamics $\mathcal{V}$, while the states $\xi$ and controls $\zeta$ are bounded as specified in Eq.~\eqref{eq:state} and ~\eqref{eq:ctr}. 
$\mathcal{T_F}$ represents the constraints on the tire forces.	
$\mathcal{E}$ imposes spatial constraints to ensure the vehicle remains within the safe envelope.	
All these terms are detailed next.

\subsubsection{Hard Constraints}: Three types of hard constraints are considered in the MPC. 
Namely, linear state and control constraints, tire force constraints, and spatial envelope constraints.

\textbf{Linear State and Control Constraints}: Three states are constrained with constant lower and upper bounds to ensure the safe operation of the vehicle:
\begin{align}
\label{eq:v_bound} v_\mathrm{min}\, \leq \, \, \, & v \leq \, v_\mathrm{max}\\
\label{eq:r_bound} r_\mathrm{min}\, \leq \, \, \, & r \leq \, r_\mathrm{max}\\
\label{eq:sa_bound} \delta_{f_\mathrm{min}}\, \leq \, \, \, & \delta_f \leq \, \delta_{f_\mathrm{max}}
\end{align}
Another constraint is applied on the acceleration state to capture the fact that as the speed increases, the maximum longitudinal acceleration the vehicle can achieve becomes smaller due to the engine power limits. 
The typical behavior of a straight-line acceleration test is illustrated in Fig.~\ref{fig:KS}-C1. 
The shape is close to a saturating exponential function, characterized by the saturating value $p_b$ and time constant $\frac{1}{p_a}$. 
Thus, the vehicle's acceleration capability is approximated as a linear constraint on $a_x$ as a function of $u_x$ as expressed below and depicted in Fig.~\ref{fig:KS}-C2:
\begin{equation}
    a_x \leq - p_a (u_x - p_b)
\end{equation}
% For the example racing cases studied in this work, the values are picked as $p_a=0.1292 \; \text{s}^{-1}$ and $p_b=60 \; \text{m/s}$ to represent the vehicle capability of the research vehicle used in this study (Toyota Research Institute 2019 LC 500).

The control bounds are set to capture the actuator limits:
\begin{align}
\label{eq:delta_f_limits}\dot{\delta}_{f_{\mathrm{min}}}\, \leq \, \, \, & \dot{\delta}_f \leq \, \dot{\delta}_{f_{\mathrm{max}}}\\
\label{eq:jx_limits}j_{x_\mathrm{min}}\, \leq \, \, \, & j_x \leq \, j_{x_\mathrm{max}}
\end{align}

\textbf{Tire Force Constraint}: As shown in Sec. \ref{sec:VD}, the friction circle is naturally embedded in Eq.~\eqref{eq:fymax} and is further approximated in Eq.~\eqref{eq:Lateral_Force}. 
Therefore, the original friction circle constraint
\begin{align}
    \label{eq:original_friction_ellipse} F_{x\bullet}^2 + F_{y\bullet}^2 &\leq \left(\mu_\bullet F_{z\bullet}\right)^2
\end{align}
can be now turned into
\begin{align}
    \label{eq:new_friction_ellipse} |F_{x\bullet}|& \leq  \mu_\bullet F_{z\bullet}
\end{align}
% Here,  $\bullet$ is a placeholder for $f$, $r$, representing front and rear axles, respectively. \Tulga{Already introduced this notation above.}
% The vehicle modeled here is an RWD vehicle, \Tulga{Already mentioned above}
Eq. \eqref{eq:new_friction_ellipse} is expanded based on whether the vehicle is accelerating or braking based on Eq.~\eqref{eq:acc_longitudinal}, \eqref{eq:br_longitudinal}, \eqref{eq:front_load} and \eqref{eq:rear_load}. 
Specifically, for the acceleration case, we have: 
\begin{equation}
        \label{eq:front_limit}
        \begin{aligned}
        0 &\leq \mu_f F_{zf}\\
        M a_x &\leq \mu_r \left(\frac{L_f}{L_f + L_r} Mg + K_{z} a_x \right)
        \end{aligned}
\end{equation}
From this equation, an upper bound for $a_x$ is derived:
\begin{equation}\label{eq:acc_upper_limit}
    a_x \leq \min\left(\frac{L_r M g}{\left(L_f + L_r\right)K_z}, \frac{\mu_r  Mg  \frac{L_f}{L_f + L_r} }{M - \mu_r  K_z} \right)
\end{equation}
Similarly, for the braking case we have: 
\begin{equation}
        \label{eq:rear_limit}\begin{aligned}
        -b_r M a_x &\leq \mu_f \left(\frac{L_r}{L_f + L_r} Mg  - K_{z} a_x \right)\\
        -(1 - b_r) M  a_x &\leq \mu_r \left(\frac{L_f}{L_f + L_r} Mg + K_{z} a_x \right)
    \end{aligned}
\end{equation}
This leads to the following lower bound for $a_x$: 
\begin{equation}\label{eq:acc_lower_limit}
    a_x  \geq \max\left(\frac{-\mu_r  Mg  \frac{L_f}{L_f + L_r}}{M (1 - b_r) + \mu_r  K_z}, \frac{-\mu_f  Mg  \frac{L_r}{L_f + L_r}  }{M  b_r - \mu_f  K_z} \right)
\end{equation}

% Despite the complexity in Eq.~\eqref{eq:acc_upper_limit} and \eqref{eq:acc_lower_limit}, these equations are finally reduced to a linear state constraint that only relates to longitudinal acceleration, which greatly improves the computational efficiency of MPC.
Thus, the original nonlinear friction circle constraint Eq. \eqref{eq:original_friction_ellipse} is reduced to a set of linear constraints on longitudinal acceleration $a_x$, greatly improving the computational efficiency of the MPC.

\begin{figure*}
\centering
\ifthenelse{\boolean{PngBool}}{%
  \includegraphics[width=0.75\textwidth]{PngVer/Figure1Methods.png}%
}{%
  \includegraphics[width=0.75\textwidth]{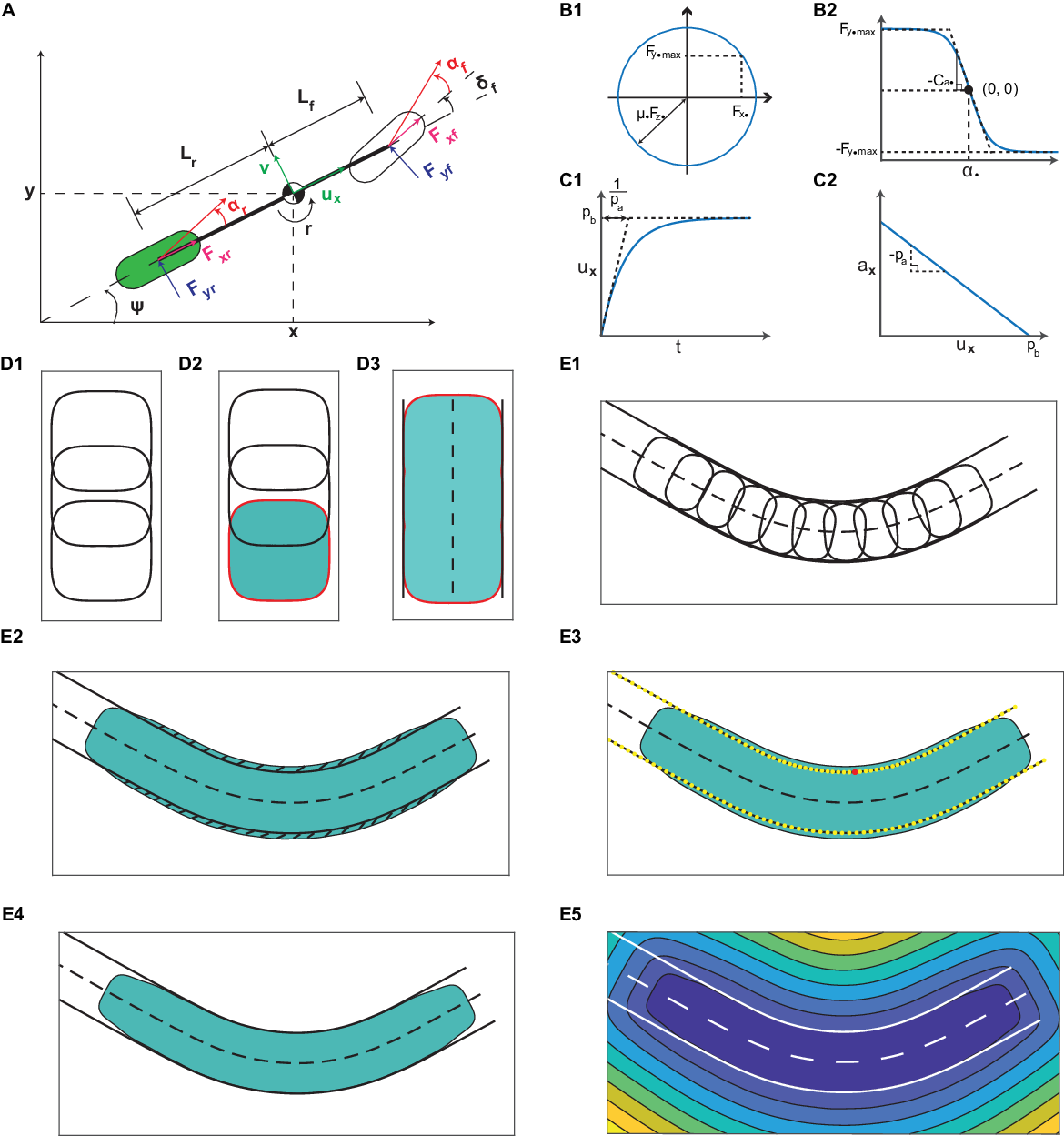}%
}
\caption{The plots for vehicle dynamics, tire forces, acceleration constraints, spatial envelope hard constraints, and spatial envelope cost (\textbf{A}) Free-body diagram of the three-DoF dynamical single-track model. The green tire means the vehicle is rear-wheel driven. (\textbf{B1}) The friction circle limit. (\textbf{B2}) The sigmoid representation of the lateral tire force. (\textbf{C1}) Longitudinal speed versus time in a straight acceleration test. (\textbf{C2}) Maximum acceleration deduced from the straight acceleration test. (\textbf{D1}) Three block constraints stacked together. (\textbf{D2}) The feasible set for one block constraint. (\textbf{D3}) The feasible set when any one of the block constraints is satisfied. The shape of the teal region forms a lane. (\textbf{E1}) A portion of the race track. The track is segmented by ten blocks. (\textbf{E2}) A contour plot of KS aggregation over the blocks, where the teal region is the feasible set. The hatched area is the critical set $\Omega_c$. (\textbf{E3}) The yellow dots are the points on lane boundaries and the red point has the minimal LogSumExp value. (\textbf{E4}) The feasible set given by the conservative spatial envelope hard constraint. (\textbf{E5}) A contour plot of the spatial envelope cost.} 
\label{fig:KS}
\end{figure*}

\textbf{Spatial Envelope Constraints}: 
The spatial envelope is defined as the region through which the vehicle can traverse safely. 
% For general purposes, there is no assumption about the shape of the envelope. 
In general, the envelope can take any shape.
Accordingly, no assumption is made about the envelope shape in this study.
Thus, methods like using a cubic function to fit the lane boundary in \cite{nuthong2010lane}  are not suitable, because they require the global X and Y positions to have a one-to-one relationship. 

In this work, the need to account for an arbitrary envelope shape is addressed by segmenting it into a sequence of blocks.
Fig.~\ref{fig:KS}-D1 illustrates an example, where a track segment is divided into three blocks. 
The constraint for the vehicle to stay within one of the blocks is defined as follows:

\begin{align}
    g^j_{b}(\xi) &= d^j_{b}(\xi) - 1\leq 0 \label{eq:nomal_block} \\
    \Delta x_b^j &= x - x_b^j \\
    \Delta y_b^j &= y - y_b^j
\end{align}
\begin{align}
\label{eq:pnorm}
\begin{split}
    d^j_{b}(\xi) &= \left( \left( \frac{\cos(\psi^j_{b}) \Delta x_b^j + \sin(\psi^j_{b}) \Delta y_b^j}{L^j_{b}} \right)^p \right. \\
                 &\quad \left. + \left( \frac{\cos(\psi^j_{b}) \Delta y_b^j - \sin(\psi^j_{b}) \Delta x_b^j}{W^j_{b}} \right)^p \right)^{\frac{1}{p}}
\end{split}
\end{align}
Here, $x$ and $y$ represent the x and y positions of the vehicle's center of gravity (C.G.) and are the entries of $\xi$. 
The variables $x_{b}^j$, $y_{b}^j, \psi_{b}^j, L^j_{b}$, and $W^j_{b}$ correspond to the block center's x-y positions, block yaw angle, half block length, and half block width, respectively. 
The superscript $j$ represents the $j^{th}$ block.
The parameter $p$ defines the p-norm, which is set to 4 to make hyperellipses. 
A higher value of $p$ can make the block more closely resemble a rectangle, but it increases the chance of causing ill-conditioning in the optimization. 
The satisfaction of Eq.~\eqref{eq:nomal_block} indicates that the vehicle is inside the specific block, shown as the teal region in Fig.~\ref{fig:KS}-D2. 
If the vehicle is inside any of these blocks, the condition is equivalent to constraining it inside the entire envelope, shown as the teal region in Fig.~\ref{fig:KS}-D3. 
The expression for the \textbf{OR} operation over multiple block constraints is discussed below.

% \begin{definition}
%     \label{def: constraint_define}   
%     The satisfaction of the $j^{th}$ condition is define as $g^j(\xi) \leq 0$. 
% \end{definition}
% \Tulga{A proposition or lemma or theorem must be very precise. 'Given multiple conditions' is a vague statement. More importantly, this is not a proposition.}\Congkai{It seems like the definition of constraint satisfaction is only used in Lemma.~\ref{lemma:or_condition}, and it has already be define in the beginning of Lemma.~\ref{lemma:or_condition}. So, I removed the Definition 1.}

\begin{lemma}
\label{lemma:or_condition}
Let there be $n$ conditions in the form of $g^j(\xi) \leq 0, j=1,2,...,n$.
Then, the OR operation, where at least one of the conditions $g^j(\xi)$ is satisfied, can be expressed as a single inequality constraint as follows:
\begin{equation}
    g_\text{min}(\xi) = \min\{g^1(\xi), g^2(\xi), ... g^{n}(\xi)  \} \leq 0 \label{eq:lemma_or_constraint}
\end{equation}
\end{lemma}

\begin{proof}
Suppose that at least one of the conditions is satisfied. 
Then, $\exists j^* \in \left[1, n\right] \text{, } g^{j^*}(\xi) \leq 0$. 
Since $g_\text{min}(\xi)$ represents the minimum value among the expressions $g^j(\xi)$, it follows that $g_\text{min}(\xi) \leq g^{j^*}(\xi) \leq 0$. \\
Conversely, assume that Eq.~\eqref{eq:lemma_or_constraint} is satisfied, but none of the conditions $g^j(\xi)$ are satisfied.
This leads to a contradiction, since $g_\text{min}(\xi) \le 0$ implies that 
$\exists j^* \in \left[1, n\right] \text{ such that } g^{j^*}(\xi) = g_\text{min}(\xi) \leq 0$. 
% According to Proposition~\ref{def: constraint_define}, this implies that the $j^*$ condition is satisfied, which contradicts the assumption that none of the conditions are satisfied. 
\end{proof}

Using Lemma~\ref{lemma:or_condition}, the expression for the vehicle staying inside the envelope is given by:
\begin{equation}
    g_\text{min, b}(\xi) = \min\{g^1_{b}(\xi), g^2_{b}(\xi), ..., g^{n_{\text{block}}}_{b}(\xi)  \} \leq 0 \label{eq:or_constraint}
\end{equation}
where $n_{\text{block}}$ is the total number of blocks. 
Note that Eq.~\eqref{eq:or_constraint} is not differentiable due to the minimization operation on the functions. 
Therefore, a smooth minimization technique is required to closely approximate $g_\text{min, b}(\xi)$ while ensuring it is twice continuously differentiable. 
% One candidate is the Boltzmann operator \cite{asadi2017alternative}:
% \begin{equation}
%     g_\text{Boltz, b}(\xi) = 
%     \frac{\sum^{n_{\text{block}}}_{j=1}g^j_{b}(\xi) e^{\rho_{b} g^j_{b}(\xi)} }{\sum^{n_{\text{block}}}_{j=1} e^{\rho_{b} g^j_{b}(\xi)}} ,~\forall \rho_b < 0\label{eq:boltzman}
% \end{equation}
% Here, $\rho_{b}$ is the tuning parameter for the Boltzmann operator. 
% With a sufficiently small $\rho_{b}$, the Boltzmann operator can closely approximate $g_\text{min, b}(\xi)$. 
% However, to avoid ill-conditioning in the optimization, $\rho_{b}$ also needs to be close to zero, which leads to an overly-conservative approximation and limits the traversable region. 
% In this regard, selecting $\rho_{b}$ is not trivial, and there is no analytical guarantee that its design will apply to arbitrary settings, e.g., such as the number and dimensions of the blocks.

One candidate is the Kreisselmeier–Steinhauser (KS) aggregation \cite{kreisselmeier1980systematic} used in constrained optimization, which is also referred to as LogSumExp (LSE) \cite{nielsen2016guaranteed} in machine learning. 
In the context of constrained optimization, KS aggregation is used to conservatively estimate the maximum. 
In the context of machine learning, the LSE function is used for both  maximum and minimum approximations. 
However, the conservativeness property is not guaranteed for the minimum approximation. 
Thus, an expression for the spatial envelope constraint with conservativeness guarantee is derived next.

\begin{lemma}
\label{lemma:conservative}
    Given two sets $\Omega^i = \{ \xi |~ g^{i}(\xi) \leq 0 \}$ and $\Omega^j = \{  \xi |~ g^{j}(\xi) \leq 0 \}$, if $g^i(\xi) \leq g^j(\xi) ~\forall \xi$, then $\Omega^j \subseteq \Omega^i$.
\end{lemma}

\begin{proof}
  By definition,  $\forall \xi \in \Omega^j$, $g^j(\xi) \leq 0$. Given that $g^i(\xi) \leq g^j(\xi)~\forall \xi$, it follows that $g^i(\xi) \leq 0 ~ \forall \xi \in \Omega^j$. Therefore, if $\xi \in \Omega^j$, then $\xi \in \Omega^i$. This implies $\Omega^j \subseteq \Omega^i$.
\end{proof}

\begin{lemma}
\label{lemma:LES_max_property}
   Given a set of functions $G = \left\{  g^j(\xi) \text{, } j \in \left[1, n \right] \right\}$, consider the function $g_\text{max}(\xi)$ defined as follows:
    \begin{equation}
        g_\text{max}(\xi) = \max\{g^1(\xi), g^2(\xi), ... ,g^{n}(\xi)  \} \leq 0 \label{eq:max_constraint}
    \end{equation}
The LSE over the set $G$ with positive $\rho_\text{LSE}$ satisfies the following properties:
    \begin{equation}
    g_\text{max}(\xi) \leq g_\text{LSE}(\xi)  \leq g_\text{max}(\xi) + \frac{\ln (n)}{\rho_\text{LSE}} \label{eq:lemma_max_lse_Inequality}
\end{equation}
where $g_\text{LSE}(\xi)$ is the expression of LSE, and the first inequality is strict unless $n = 1$.

\begin{equation}
   g_\text{LSE}(\xi) = \frac{1}{\rho_\text{LSE}} \ln \left( \sum_{j=1}^{n} e^{\rho_\text{LSE} \; g^j(\xi)} \right), \forall \rho_\text{LSE} >0 \label{eq:lemma_max_les_agg}
\end{equation}
and $\rho_\text{LSE}$ is the tuning parameter. 
\end{lemma}

\begin{proof}
% Let \( g_\text{LSE}(\xi) \) be defined as
% \begin{equation}
% g_\text{LSE}(\xi) = \frac{1}{\rho_\text{LSE}} \ln \left( \sum_{j=1}^{n} e^{\rho_\text{LSE} g^j(\xi)} \right)
% \end{equation}
% with \( \rho_\text{LSE} > 0 \), and let \( g_{\max}(\xi) = \max_{j \in [1,n]} g^j(\xi) \). 
Define \( M = g_{\max}(\xi) \) and rewrite the sum inside the logarithm in \eqref{eq:lemma_max_les_agg}:
\begin{align}
g_\text{LSE}(\xi) 
&= \frac{1}{\rho_\text{LSE}} \ln \left( \sum_{j=1}^{n} e^{\rho_\text{LSE} \; g^j(\xi)} \right) \notag \\
&= \frac{1}{\rho_\text{LSE}} \ln \left( \sum_{j=1}^{n} \left( e^{\rho_\text{LSE} (g^j(\xi) - M)} e^{\rho_\text{LSE} M} \right ) \right) \notag \\
&= \frac{1}{\rho_\text{LSE}} \left[ \ln \left( \sum_{j=1}^{n} e^{\rho_\text{LSE}(g^j(\xi) - M)} \right) + \rho_\text{LSE} M \right] \notag \\
&= M + \frac{1}{\rho_\text{LSE}} \ln \left( \sum_{j=1}^{n} e^{\rho_\text{LSE}(g^j(\xi) - M)} \right)
\end{align}
Since \( g^j(\xi) - M \leq 0 \) for all \( j \), and there exists at least one \( j \) such that \( g^j(\xi) = M \), we have:
\begin{itemize}
    \item \( e^{\rho_\text{LSE}(g^j(\xi) - M)} \in (0, 1] \)
    \item \( \sum_{j=1}^{n} e^{\rho_\text{LSE}(g^j(\xi) - M)} \in [1, n] \)
\end{itemize}
Thus,
\begin{equation}
\label{eq:log_sum_exp_proof}
\ln \left( \sum_{j=1}^{n} e^{\rho_\text{LSE}(g^j(\xi) - M)} \right) \in [0, \ln n]
\end{equation}
Note that Eq.~\eqref{eq:log_sum_exp_proof} becomes 0 only when $n = 1$. Based on Eq.~\eqref{eq:log_sum_exp_proof}, the following inequalities are derived.

\begin{equation}
g_{\max}(\xi) \leq g_\text{LSE}(\xi) \leq g_{\max}(\xi) + \frac{\ln n}{\rho_\text{LSE}}
\end{equation}
which proves inequality~\eqref{eq:lemma_max_lse_Inequality}. 
The closed-form expression for \( g_\text{LSE}(\xi) \) is given directly in~\eqref{eq:lemma_max_les_agg} by definition. 
\end{proof}

% \begin{proof}
%     The proof is omitted for brevity but can be found in \cite{poon2007adaptive, raspanti2000new}.
% \Tulga{If this is a lemma we borrow from another paper, we should add citation after the lemma heading. Sometimes, people even write 'Lemma 3 (Lemma X of [ref]'.} \Congkai{Added} \Tulga{Where is it added? I don't see it.}
% \end{proof}

\begin{lemma}
\label{lemma:KS_property}
   Given the function set defined in Lemma~\ref{lemma:LES_max_property} and $g_\text{min}$ defined in Eq~\eqref{eq:lemma_or_constraint}, the LSE over the set $G$ with negative $\rho_\text{LSE}$ satisfies the following properties:
       \begin{align}
    g_\text{min}(\xi)  & + \frac{\ln(n)}{\rho_\text{LSE}} \leq {g}_\text{LSE}(\xi) \leq g_\text{min}(\xi) \label{eq:lemma_lse_property}\\
    g_\text{LSE}(\xi) & = \frac{1}{\rho_\text{LSE}} \ln \left( \sum_{j=1}^{n} e^{\rho_\text{LSE}(g^j(\xi))} \right)\ , ~ \forall \rho_\text{LSE} < 0 \label{eq:lemma_LSE}
    \end{align}

    where the second inequality in Eq.~\eqref{eq:lemma_lse_property} is strict unless $n = 1$.
% where $g_\text{LSE}(\xi)$ is the expression of LSE, and $\rho_\text{LSE}$ is the tuning parameter. 
\end{lemma}

\begin{proof}
    Given the function set $G$ and a negative $\rho_\text{LSE}$, define the negated function set $\Tilde{G}$, and the negated tuning parameter $\Tilde{\rho}_\text{LSE}$:
    \begin{align}
        \Tilde{G} &= \left\{  \Tilde{g}^j(\xi) \text{, } j \in \left[1, n \right] \right\}\\
        \Tilde{g}^j(\xi) &= -g^j(\xi) \text{, } j \in \left[1, n \right] \label{eq:lemma:negate} \\
        \Tilde{\rho}_\text{LSE} &= - \rho_\text{LSE} > 0
    \end{align}
    
    Based on Lemma~\ref{lemma:LES_max_property}, the LSE over $\Tilde{G}$ obeys the following inequalities:
    \begin{align}
        \Tilde{g}_\text{max}(\xi)  \leq \Tilde{g}_\text{LSE}(\xi) & \leq \Tilde{g}_\text{max}(\xi) + \frac{\ln (n)}{\Tilde{\rho}_\text{LSE}}, \forall \Tilde{\rho}_\text{LSE} >0 \label{eq:lemma_ks_min_derive_property}\\
        \Tilde{g}_\text{LSE}(\xi) & = \frac{1}{\Tilde{\rho}_\text{LSE}} \ln \left( \sum_{j=1}^{n} e^{\Tilde{\rho}_\text{LSE}(\Tilde{g}^j(\xi))} \right) \label{eq:lemma_ks_min_derive_agg} \\
        \Tilde{g}_\text{max}(\xi) & = \max\{\Tilde{g}^1(\xi), \Tilde{g}^2(\xi), ... \Tilde{g}^{n}(\xi)  \} \label{eq:lemma_ks_min_derive_min_define}
    \end{align}
    
    Note that the first inequality in Eq.~\eqref{eq:lemma_ks_min_derive_property} is strict unless $n = 1$.
    Based on Eq.~\eqref{eq:lemma:negate}, the following relationship is derived:
    \begin{equation}
        \Tilde{g}_\text{max}(\xi) = -g_\text{min}(\xi) \label{eq:lemma_ks_min_derive_maxmin}
    \end{equation}
    Based on Eq.~\eqref{eq:lemma_ks_min_derive_maxmin} and Eq.~\eqref{eq:lemma_ks_min_derive_property}, the property of LSE over $G$ with negative $\rho_\text{LSE}$ is derived:
    \begin{equation}
        g_\text{min}(\xi) + \frac{\ln(n)}{\rho_\text{LSE}} \leq -\Tilde{g}_\text{LSE}(\xi) = {g}_\text{LSE}(\xi) \leq g_\text{min}(\xi) \label{eq:lemma_ks_min_derive_min_property}
    \end{equation}
    where the second inequality is strict unless $n = 1$.
    
    % \Congkai{There is no proof for the inequalities of smooth minimization (LSE with negative $\rho$) except for Wikipedia and Methematics forum. So, here I put our own derivation.}
\end{proof}

Similar proofs for Lemma~\ref{lemma:LES_max_property} and \ref{lemma:KS_property} have been presented in \cite{molnar2023composing} for control barrier function constraint aggregation. Nevertheless, these steps are retained here for completeness and consistent notation throughout the paper.

% \begin{lemma}
% \label{lemma:KS_property}
%    Given a set of functions $G = \left\{  g^j(\xi) \text{, } j \in \left[1, n \right] \right\}$, the LSE over the set $G$ satisfies the following properties:
%        \begin{align}
%     g_\text{min}(\xi)  & - \frac{\ln(n)}{\rho_\text{LSE}} \leq {g}_\text{LSE}(\xi) < g_\text{min}(\xi) \label{eq:lemma_lse_property}\\
%     g_\text{LSE}(\xi) & = \frac{1}{\rho_\text{LSE}} \ln \left( \sum_{j=1}^{n} e^{\rho_\text{LSE}(g^j(\xi))} \right)\ , ~ \forall \rho_\text{LSE} < 0 \label{eq:lemma_LSE}
%     \end{align}
% where $g_\text{LSE}(\xi)$ is the expression of LSE, and $\rho_\text{LSE}$ is the tuning parameter. 
% \end{lemma}

% \begin{proof}
%     The proof is omitted for brevity, but can be found in \cite{nielsen2016guaranteed}.
%     \Tulga{If this is a lemma we borrow from another paper, we should add citation after the lemma heading. Sometimes, people even write 'Lemma 3 (Lemma X of [ref]'.}
% \end{proof}

% \begin{prop}
%     \label{prop:set_define}
%     The safety set is given as $\Omega_s$. The feasible set $\Omega_0$ is defined as $\Omega_0 = \{  \xi |~ 
% g_\text{LSE}(\xi) \leq 0 \}$. The critical set is defined as $\Omega_c = \Omega_s^C \cap \Omega_0$.
% \Tulga{These are all definitions, not propositions. A Proposition is like a Lemma or Theorem; it is a statement that can be proved to be right or wrong. These definitions can be merged into the Lemma below.}
% \end{prop}

\begin{lemma}
\label{lemma:subset}
    Consider the safety set $\Omega_s$ defined as $\Omega_s = \{  \xi |~g_\text{min}(\xi) \leq 0 \}$; the feasible set $\Omega_0$ defined as $\Omega_0 = \{  \xi |~g_\text{LSE}(\xi) \leq 0 \}$; and the critical set  defined as $\Omega_c = \Omega_s^C \cap \Omega_0$.
    If $\Omega_c \neq \varnothing$, then $\Omega_0 \not\subseteq \Omega_s$. Otherwise, $\Omega_0 \subseteq \Omega_s$.
\end{lemma}

\begin{proof}
    If $\Omega_c \neq \varnothing$, then $\exists \xi \in \Omega_0$ such that $\xi \notin \Omega_s$, implying $\Omega_0 \not\subseteq \Omega_s$.\\
    If $\Omega_c = \varnothing$, then $\forall \xi \in \Omega_0$, $\xi \in \Omega_s$, implying $\Omega_0 \subseteq \Omega_s$.
\end{proof}

\begin{theorem}
\label{theorem:con_LSE}
Let $\Omega_s$, $\Omega_0$, and $\Omega_c$ be defined as in Lemma \ref{lemma:subset}. Let $\Tilde{\Omega}_0$ be defined as
\begin{align}
    \Tilde{\Omega}_0 &= \{ \xi |~ g_\text{LSE}(\xi) - \epsilon_0 < 0 \} \\
    \epsilon_0 &= \underset{\xi \in \Omega_c}{\min}g_\text{LSE}(\xi) \label{eq:epsilon}
\end{align}
Then, if $\Omega_c \neq \varnothing$, then $\Tilde{\Omega}_0 \subseteq \Omega_s$.
\end{theorem}

\begin{proof}
    By the definitions of $\Omega_c$ and $\Omega_0$, $\forall \xi \in \Omega_c$, $\xi \in \Omega_0$, and thus $\forall \xi \in \Omega_c ,~ g_\text{LSE}(\xi) \leq 0$. 
    This implies that $\epsilon_0 \leq 0$ in Eq.~\eqref{eq:epsilon}. 
    Based on Lemma~\ref{lemma:conservative}, $\Tilde{\Omega}_0 \subseteq \Omega_0$ because $g_\text{LSE}(\xi) \leq g_\text{LSE}(\xi) - \epsilon_0$. 
    Define a new critical set as $\Tilde{\Omega}_c = \Omega_s^C \cap \Tilde{\Omega}_0$. 
    The new set satisfies $\Tilde{\Omega}_c \subseteq \Omega_c$ because $\Tilde{\Omega}_0 \subseteq \Omega_0$. 
    According to the definition of $\epsilon_0$ in Eq.~\eqref{eq:epsilon}, $\epsilon_0$ is the minimal LSE value in the set $\Omega_c$, implying $g_\text{LSE}(\xi) \geq \epsilon_0 ~ \forall \xi \in {\Omega}_c$. 
    Since $\Tilde{\Omega}_c \subseteq \Omega_c$, it follows that $\forall \xi \in \Tilde{\Omega}_c$, $\xi \in {\Omega}_c$, and thus $g_\text{LSE}(\xi) \geq \epsilon_0 ~ \forall \xi \in \Tilde{\Omega}_c$. 
    Since $(g_\text{LSE}(\xi) \geq \epsilon_0) \cap (g_\text{LSE}(\xi) < \epsilon_0) = \varnothing$, it follows that $\Tilde{\Omega}_c \cap \Tilde{\Omega}_0 = \varnothing$. 
    If $\Tilde{\Omega}_0 = \varnothing$, then $\Tilde{\Omega}_0 \subseteq \Omega_s$. 
    Otherwise, $\Tilde{\Omega}_c = \varnothing$, and according to Lemma~\ref{lemma:subset}, $\Tilde{\Omega}_0 \subseteq \Omega_s$.
\end{proof}

Theorem~\ref{theorem:con_LSE} proves that subtracting $\epsilon_0$ ensures the new feasible set is strictly contained within the safety set, which is essential for maintaining safety. 
The value $\epsilon_0$ is defined as the minimum of the LSE function over the critical set. 
The critical set is defined as the region outside the safety set but within the original feasible set specified by the LSE function.

A portion of a racetrack is shown in Fig.~\ref{fig:KS}-E1 through E4  to demonstrate the conservativeness of the spatial envelope constraint. 
In Fig.~\ref{fig:KS}-E1, the blocks are used to segment the envelope. 
Each block overlaps with at least one other block to ensure connectivity. 
Based on Lemma~\ref{lemma:KS_property}, the LSE is used to approximate the function $g_\text{min,b}(\xi)$ defined by Eq.~\eqref{eq:or_constraint}
as follows:
\begin{align}
    \label{eq:lse_block} {g}_\text{lse,b}(\xi) &= \frac{1}{\rho_\text{LSE}} \ln \left( \sum_{j=1}^{n_{\text{block}}} e^{\rho_\text{LSE}(g^j_{b}(\xi))} \right) , ~\forall \rho_\text{LSE} < 0
\end{align}

The feasible region defined by ${g}_\text{lse,b}(\xi) \leq 0$ is depicted as the teal region in Fig.~\ref{fig:KS}-E2. 
Although the blocks are within the lane boundaries, the feasible region described by Eq.~\eqref{eq:lse_block} extends beyond these boundaries. 
According to Theorem~\ref{theorem:con_LSE}, the critical set $\Omega_c$ must be identified first. 
The critical set lies outside the lane boundaries but within the feasible set, as illustrated by the hatched area in Fig.~\ref{fig:KS}-E2. 
Within $\Omega_c$, the minimal LSE value $\epsilon_0$ is identified, and the minimizer is marked with a red dot in Fig.~\ref{fig:KS}-E3. 
Note that in this work, in all scenarios considered, the turning of the lane is not sharp, and the lane width is large enough. 
By construction, the LSE attains its minima on the safety boundary (lane boundaries); hence the minimizer over the critical set lies on that boundary segment.
Therefore, in this work, only the lane boundary points are evaluated to find the minimal LSE value. 
These boundary points are depicted as yellow dots in Fig.~\ref{fig:KS}-E3. 
Then, the conservative spatial envelope constraint is expressed as follows:
\begin{align}
    \label{eq:lse_eq} {g}_{\text{envelope}}(\xi) &= {g}_\text{lse,b}(\xi) - \epsilon_0 < 0
\end{align}
The resultant feasible set is shown as the teal region in Fig.~\ref{fig:KS}-E4. 
The feasible region lies within the lane boundaries, thereby illustrating the conservativeness of the developed spatial envelope constraint.

During optimization, the states are treated as discrete rather than continuous design variables. 
Consequently, the $i^{th}$ prediction time is denoted as $\mathcal{T}(i)$, and the corresponding state is expressed as $\xi(\mathcal{T}(i))$.
Thus, the following conservative spatial envelope constraint is applied in the optimization.
\begin{align}
    \label{eq:envelope_cons} {g}_{\text{envelope}}(\mathcal{T}(i)) &= \frac{1}{\rho_\text{LSE}} \ln \left( \sum_{j=1}^{n_{\text{block}}} e^{\rho_\text{LSE}(g^j_{b}(\xi(\mathcal{T}(i))))} \right) - \epsilon_0 < 0
\end{align}
\subsubsection{Cost Function}\label{sec:COST}: The cost function in Eq.~\eqref{eq:cost} is further expanded as:
\begin{equation}
\label{eq:genericCostFunction}
    J = J_{\text{state}} + J_{\text{control}} + J_{\text{envelope}} + J_{\text{specific}}
\end{equation}
\textbf{State Cost ($J_{\text{state}})$} is the cost term for vehicle states:
\begin{equation}
    \label{eq:state_cost}
    J_{\text{state}} = \int^{t_f}_{t_0} \left(
    w_{\delta_f} {\delta_f}^2 + w_{a_x} a_x^2 + w_{v} v^2 + w_{\kappa} \left(\frac{r}{u_x}\right)^2
    \right)
    dt
\end{equation}
The first two terms regulate the control inputs (steering, acceleration) to discourage excessive steering and acceleration, as in racing the vehicle should drive as straight and as close to its top speed as possible to minimize the lap time. 
The third term aims to minimize the vehicle's lateral velocity to prevent the vehicle from sliding as a safety and stability consideration.
The last term minimizes the path curvature, again in line with the goal of driving as straight as possible to allow for maximum speed.

\textbf{Control Cost ($J_{\text{control}})$} is expressed as: 
\begin{equation}
    \label{eq:control_cost}
    J_{\text{control}} = \int^{t_f}_{t_0} \left(
    w_{\dot{\delta_f}} \dot{\delta_f}^2 + w_{j_x} j_x^2
    \right)
    dt
\end{equation}
The two terms in the expression minimize the rate of change of control inputs to ensure smooth steering and acceleration.

\textbf{Spatial Envelope Cost ($J_{\text{envelope}}$)} brings the vehicle inside the envelope to ensure safety. 
It is given as:
\begin{align}
\label{eq:hyperbolic} J_{\text{envelope}} &= \sum_{i=1}^{n_\text{ds}} w_{\text{envelope}}  \log\left(1 + e^{ - \theta_\text{hp}  (g_{sm} + g_{\text{envelope}}(\mathcal{T}(i)) } \right)
\end{align}
Here, $n_\text{ds}$ is the number of collocation points and $\theta_\text{hp}$ is the hyper-parameter for the softplus function. 
By employing a softplus function, the cost is designed to remain approximately constant when the vehicle is within the envelope and to increase linearly with respect to the distance away from the envelope when the vehicle is outside the envelope, similar to \cite{karino2023shared}.	
This is illustrated by the contour plot of the cost along a portion of the racetrack in Fig. \ref{fig:KS}-E5. 
As long as the vehicle remains within the racetrack, the cost is approximately zero.
Conversely, when the vehicle veers off the track, the cost increases linearly as shown by the color change on the plot.
This cost promotes the planned trajectories to remain within the envelope.

\textbf{Specific Cost ($J_{\text{specific}}$)} is a placeholder where costs related to specific scenarios can be included. The detailed expression of this type of cost is discussed in later sections when specific scenarios are studied.

\subsubsection{Solving Strategy}
In all scenarios tested in this work, the MPC framework plans for $T_p =6.75$ s into the future with 25 collocation points, resulting in 24 time intervals denoted as $T_1, T_2, ..., T_{24}$. 
These values were selected via a sweep over $T_p$ and collocation grid densities to balance look-ahead, real-time feasibility, and numerical integration error.
The time intervals are distributed as:
\begin{equation}
    T_i = \begin{cases}
        0.15 & \text{if} \quad i \leq 15 \\
        0.5 & \text{otherwise}
    \end{cases}, \; i=1,..., 24
\end{equation}
with $i$ being an index for the time interval.
The shorter time interval at the beginning of the prediction horizon ensures that the propagation error does not become significant enough to jeopardize vehicle operation. 
A longer time horizon at the end allows the MPC to look further into the future without incurring a high computational cost, facilitating earlier responses to upcoming turns.	
This approach offers an advantage over using the same time interval throughout the prediction horizon, as the latter either fails to provide adequate foresight or imposes additional computational burden on the solver to achieve the same propagation accuracy.	
The backward Euler integration scheme is employed between each time step due to its computational efficiency and unconditional stability.	
The nonlinear optimization problem is converted into a nonlinear program using the JuMP package in Julia and solved with Ipopt \cite{wachterIpopt}.
This implementation enables the algorithm to plan trajectories at a frequency of at least 10 Hz.

\subsection{Spatial Envelope Design} 
% \Yufei{Envelope planning methodology goes here}

The problem of designing the blocks that form a drivable spatial envelope poses a trade-off.
On the one hand, to ensure fast computation of the LSE function, a small number of blocks is desired.
On the other hand, to cover the drivable region as accurately as possible, a large number of small blocks may be needed.

This trade-off problem is simplified by breaking it into two subproblems: designing an optimal block for segments of the given road boundary recursively, and combining the designed blocks. 

For designing a single block, possible solutions include optimization and machine learning. 
While optimization produces a high quality result at the cost of a longer solving time, machine learning methods are often faster, but the solution quality is not guaranteed. 
To address the drawbacks of both solutions,  this work proposes a combined approach as illustrated in \ref{fig:envBlockDiagram}. 
The envelope is constructed by sequentially designing and connecting individual blocks. Initially, road boundary data are segmented into overlapping regions. 
For each region, reinforcement learning (RL) is used to provide a general estimate of block parameters $L, W, C, \theta$, which refer to block half length, half width, left center, and block yaw angle. 
These estimates are then refined through an optimizer, maximizing the block's area inside the road segment. 
The start point of the next road segment is set to be the center of the designed block. 
The blocks are thus planned recursively until the end of the road boundary data.

\begin{figure}
    \centering
    \includegraphics[width=0.49\textwidth]{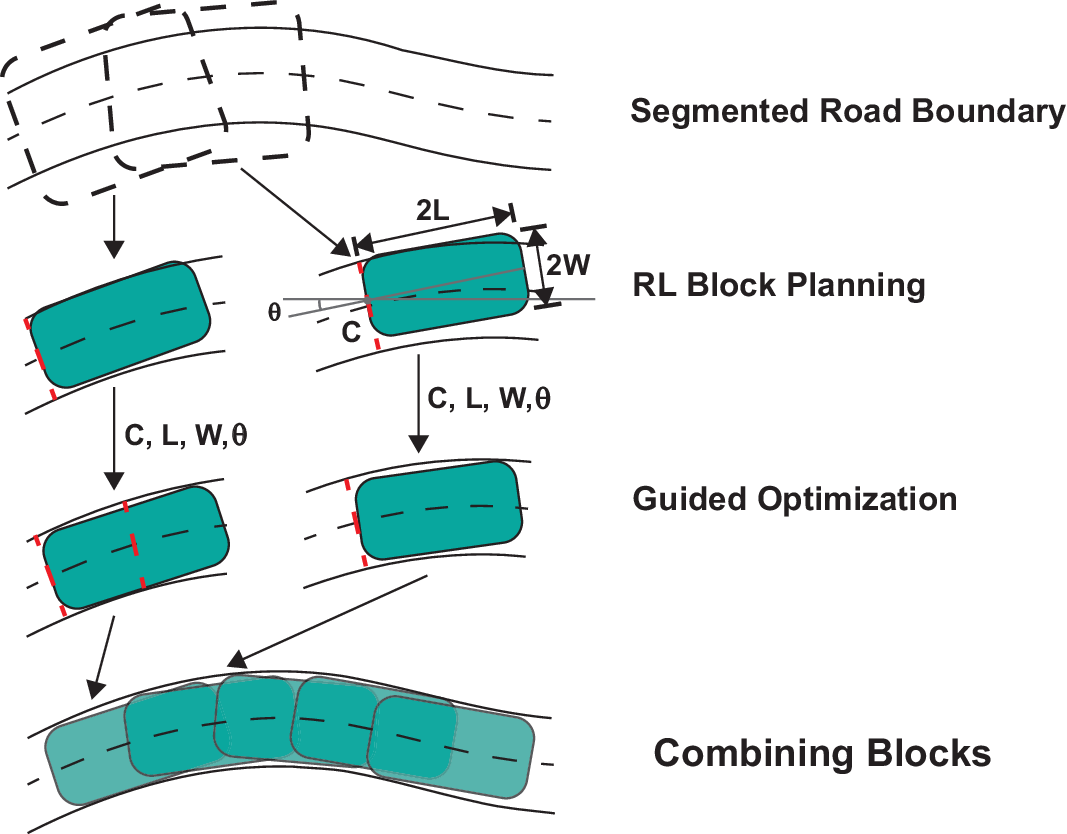}
    \caption{Spatial envelope planning framework. Block left center, yaw angle, half length, and half width are passed from RL to optimizer as initialization parameters.}
    \label{fig:envBlockDiagram}
\end{figure}

\subsubsection{Reinforcement Learning Based Policy Optimization}
This work proposes using Proximal Policy Optimization (PPO) for spatial envelope planning policy training \cite{schulman2017proximal}.

\textbf{Road Generator}: The observation data for the RL agent is generated using a road generator. 
This generator first randomizes a curvature array $\kappa_i$, offset by another random value $\epsilon_i$. 
Based on the curvatures, the centerline of the road $C_{i}$ is formed as follows:
\begin{align}
    C_{0,x}&=0, \;
    C_{i,x}=C_{i-1,x}+\cos{\kappa_i}\\
    C_{0,y}&=0, \;
    C_{i,y}=C_{i-1,y}+\sin{\kappa_i}
    \label{eq:block_centerline}
\end{align}
To avoid unrealistic variations in road width, five random widths between the given maximum and minimum road widths are assigned to random indices. 
The widths at other indices are interpolated linearly between these initial values. 
The width array is then smoothed as $w_i$ using a Savitzky–Golay filter \cite{savitzky1964smoothing}.
The left and right boundary points $l_i, r_i$ are computed using the following equations:
\begin{align}
    l_{i,x}&=C_{i,x}-w_i \sin{\kappa_i}\\
    l_{i,y}&=C_{i,y}+w_i \cos{\kappa_i}\\
    r_{i,x}&=C_{i,x}+w_i \sin{\kappa_i}\\
    r_{i,y}&=C_{i,y}-w_i \cos{\kappa_i}
\end{align}
\textbf{Reward Function}: The reward function is designed to encourage the designed block to maximize the area within the road boundaries while penalizing any portion that extends beyond the boundaries. 
The road is segmented into quadrilaterals $Q_i$, each formed by two points from the left boundary and two points from the right boundary. 
The reward function $\mathcal{\chi}$ is expressed as:
\begin{equation}
    \mathcal{\chi}=(A_\text{in}-A_\text{out})2L
\end{equation}
where $A_\text{in}$ is the area of the designed block within the road boundaries, $A_\text{out}$ is the area outside the boundaries, and $L$ is the half length of the block. 
This formulation encourages longer blocks.

For simplicity, each block is approximated as a rectangle with the same left center, half width, half length, and yaw angle. 
This approximation is conservative, since the area of the 4th order norm block is fully included inside the approximated rectangle.

% \begin{figure}
%     \centering
%     \includegraphics[width=0.4\textwidth]{F2blockConcept.eps}
%     \caption{Characterization of a block. Red dot is the left center point. Brown outline shows the contour for the 4-norm. Black outline is the approximated rectangle.  \Tulga{increase line thicknesses for better visibility of colors. Does the block need to be this big?}}
%     \label{fig:envRectConcept}
% \end{figure}

\subsubsection{Policy Guided Block Optimization}

The final block design is produced by solving an optimization problem to maximize the block size. 
The optimization formulation is as follows:
\begin{equation}
\begin{array}{rl}
\text{maximize} & L W \\[3pt] %adjust the number to adjust the vertical space
\text{s.t.} & g_i(\mathbf{x}) \leq 0, \quad i = 1, 2, \dots, m \\ 
            & h_j(\mathbf{x}) = 0, \quad j = 1, 2, \dots, p
\end{array}
\label{eq:envOptimization}
\end{equation}
Block parameters $L$ and $W$ are re-scaled by a factor and then used to initialize the optimization.
The factor is less than 1 to prevent infeasible initial conditions. 
Please refer to \ref{sec:env_planning_result} for the values. 

\subsubsection{Iterative Block Design}
The spatial envelope for driving is generated by iteratively designing the blocks.
The starting position of the last block is set to be on the road cross-section line passing through a point $S_{k}$ defined by the last block.
$S_k$ is $L$ ahead of the last left center point in the length direction. 

\section{Validation of the Vehicle Model} \label{sec:model_fidelity}

\begin{figure*}
\centering
\ifthenelse{\boolean{PngBool}}{%
    \includegraphics[width=0.93\textwidth]{PngVer/Fig3ModelFidelity.png}
}{%
    \includegraphics[width=0.93\textwidth]{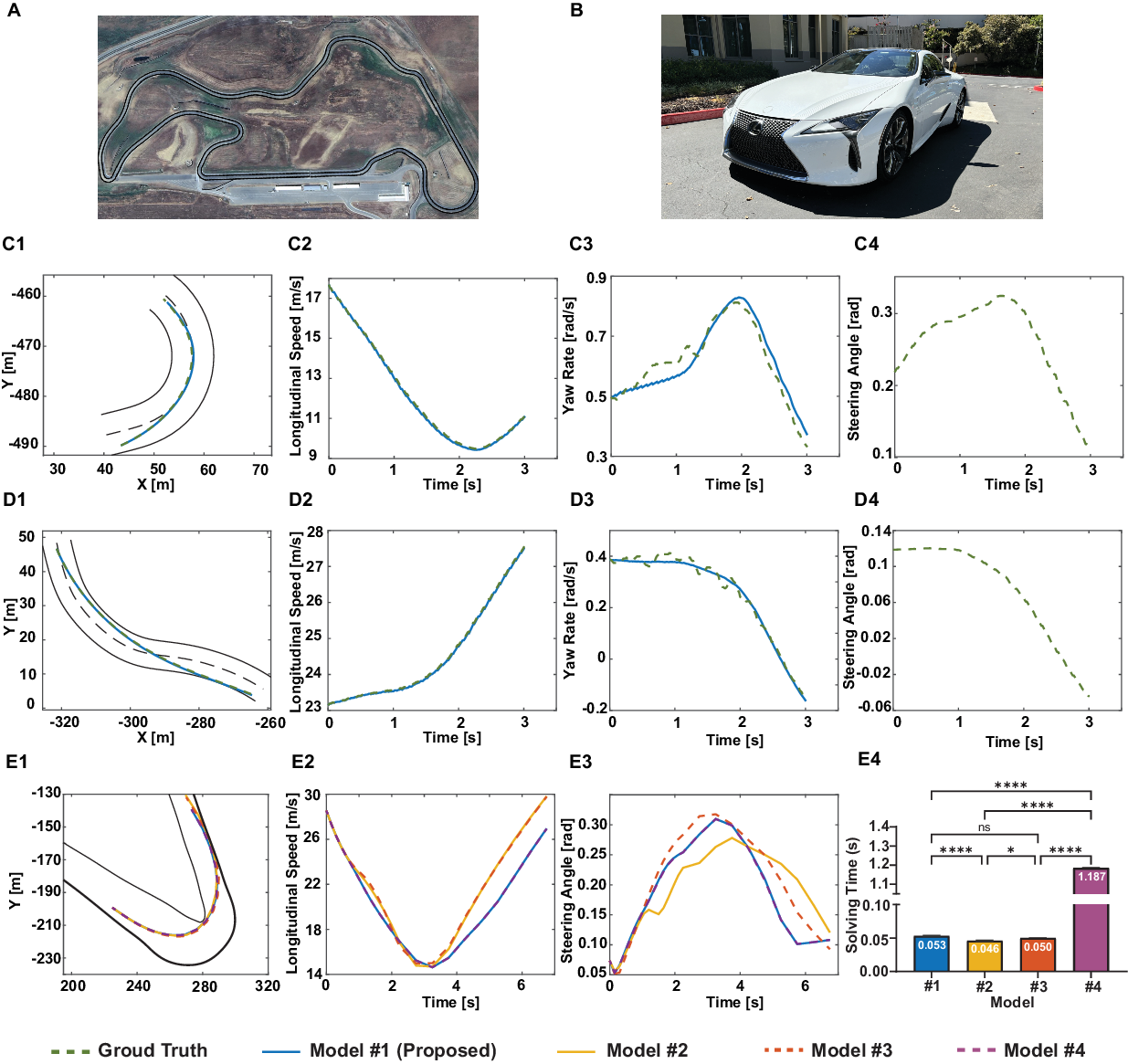}
}
\caption{The results in Sec.~\ref{sec:model_fidelity}. (\textbf{A}) An aerial view of Thunderhill West from the satellite. (\textbf{B}) A snapshot of Toyota Research Institute 2019 LC 500. (\textbf{C1} - \textbf{D4}) The plots comparing the proposed model with the real-vehicle data in open-loop simulation. The green dashed line represents the real-vehicle data and the blue solid line represents the proposed model. The subplots \textbf{C1} - \textbf{C4} show the results in a hairpin turn, while the subplots \textbf{E1} - \textbf{E4} show the results in a high-speed turn. (\textbf{C1} \& \textbf{E1}) The X-Y path comparisons in the hairpin and high-speed turn, respectively. (\textbf{C2} \& \textbf{E2}) The longitudinal speed comparisons in the hairpin and high-speed turn, respectively. (\textbf{C3} \& \textbf{E3}) The yaw rate comparisons in the hairpin and high-speed turn, respectively. (\textbf{C4} \& \textbf{E4}) The steering angle commands used by the real vehicle in the hairpin and high-speed turn, respectively. (\textbf{E1 - \textbf{E4}}) The comparisons of four vehicle models in one-time optimization. The blue solid line presents Model~\#1, which is referred to as the proposed model. The yellow solid, red dashed, and magenta dashed lines represent the benchmark Models~\#2, \#3, and \#4, respectively. (\textbf{E1}) The comparisons in optimal X-Y paths. (\textbf{E2}) The comparisons in optimal longitudinal speeds. (\textbf{E3}) The comparisons in optimal steering angles. (\textbf{E4}) Statistical comparisons in solving times. }
\label{fig:ModelFidelity}
\end{figure*}

This section evaluates vehicle model fidelity through two studies. 
In the first study, the proposed vehicle model is compared to real-vehicle data in  open-loop simulations to validate its accuracy in prediction. 
The second study establishes the effectiveness of the proposed model as a prediction model for optimal control by comparing its performance against three benchmark models.

In the first study, a real vehicle is tested at Thunderhill West, an aerial view of which shown in Fig.~\ref{fig:ModelFidelity}-A. 
The vehicle is a Toyota Research Institute 2019 LC 500, shown in Fig.~\ref{fig:ModelFidelity}-B. 
Vehicle telemetry data were collected by manually driving along the track, and two regions were selected for open-loop comparisons: one at a hairpin  turn (Fig.~\ref{fig:ModelFidelity}-C1) and the second at a high-speed turn (Fig.~\ref{fig:ModelFidelity}-D1). 
In the open-loop simulation, the initial states of the proposed vehicle model (Model \#1) are set identical to the initial states of the real vehicle, and the control inputs are linearly interpolated based on the time series from the recorded driver inputs. 
From Fig.~\ref{fig:ModelFidelity}-C1 to D4, the green dashed line represents the trajectory of the real vehicle, while the blue solid line represents that of the proposed model.
Based on Fig.~\ref{fig:ModelFidelity}-C1, the difference in path between the real vehicle and the proposed model is capped at 45 cm.
The longitudinal speeds are identical, as shown in Fig.~\ref{fig:ModelFidelity}-C2, since the control input of longitudinal jerk is derived from the speed data. 
Comparisons in yaw rate are shown in Fig.~\ref{fig:ModelFidelity}-C3. 
Noticeable differences in yaw rate might be due to the omission of chassis suspension in the proposed model.	
In this model, the roll and pitch angles are assumed to be zero. 
On the contrary, these angles can change in reality, especially in the hairpin where lateral acceleration is large. 
The steering angle of the real vehicle is shown in Fig.~\ref{fig:ModelFidelity}-C4.
The second control input of steering rate for the proposed model is derived from this data. 
Similar comparisons in the high-speed turn are shown from Fig.~\ref{fig:ModelFidelity}-D1 to D4. 
In the high-speed turn, the model predicts the global locations accurately with a maximum error of 61 cm.
Compared to the performance in the hairpin turn, the yaw rates are closer perhaps due to the suspension being less excited in the high-speed turn than in the hairpin turn. 
In the real vehicle data, the small oscillations observed at the beginning are at a frequency of about 2.5 Hz and likely correspond to the suspension \cite{rajeev2019pratheeksudi}.

In the second study, four models are used as prediction models for a one-time optimization. 
Model~\#1 is the proposed vehicle model. 
Model~\#2 considers longitudinal load transfer with the pure-slip tire model. 
Model~\#3 ignores longitudinal load transfer and the friction circle limit. 
Model~\#4 represents the model used in \cite{talbot2023shared}. 
Models~\#2 and~\#3 are included as an ablation study.
In \cite{talbot2023shared}, both longitudinal load transfer and the friction circle limit are considered. 
However, unlike Eq.~\eqref{eq:acc_upper_limit} and Eq.~\eqref{eq:acc_lower_limit}, the expression for load transfer is not derived as linear constraints.
Instead, Model \#4 adheres to the friction circle constraint given by Eq. \eqref{eq:original_friction_ellipse}.
% $F_{x\bullet}^2 + F_{y\bullet}^2 \leq \left(\mu_\bullet F_{z\bullet}\right)^2$
Furthermore, the softplus function of Eq.~\eqref{eq:Lateral_Force} is not used to calculate the maximum lateral tire force. 

The optimal trajectories of each model are shown in Fig.~\ref{fig:ModelFidelity}-E1 to E3. 
Model~\#1 and Model~\#4 exhibit identical optimal trajectories, as depicted in Fig.~\ref{fig:ModelFidelity}-E1 to E3, since they share the same equations of motion. 
Model~\#2 and Model~\#3 neglect the friction circle limit, resulting in higher mobility predictions compared to Model \#1 and \#4.
On the other hand, these two models show similar paths, and their predicted paths are longer than those of Model~\#1 and Model~\#4 in Fig.~\ref{fig:ModelFidelity}-E1. 
This behavior is explained by the similar predicted longitudinal speeds of Model~\#2 and Model~\#3, which are larger than those of Model~\#1 and Model~\#4, as shown in Fig.~\ref{fig:ModelFidelity}-E2. 
In short, the vehicle is predicted to have higher speed during corners.

By including the load transfer and neglecting the friction circle limit, Model~\#2 can have larger maximum lateral tire forces on the front wheels than on the rear wheels during deceleration, which increases the likelihood of oversteer due to load transfer. 
Meanwhile, Model~\#2 has greater maximum lateral tire forces on the rear wheels than on the front wheels during acceleration, which increases the likelihood of understeer.
These behaviors explain why the steering angle of Model~\#2 is smaller than that of Model~\#1 and Model~\#4 before 3 seconds and larger than the other two models afterward, as shown in Fig.~\ref{fig:ModelFidelity}-E3. 

Model~\#3 neglects both load transfer and the friction circle, so the imbalance in maximum lateral tire forces between the front and rear tires is less significant than in Model~\#2, making the vehicle closer to neutral-steer. 
As a result, the steering angle of Model~\#3 is closer to the proposed model than to Model~\#2, as shown in Fig.~\ref{fig:ModelFidelity}-E3. 

Pairwise comparisons of solving times, conducted as post-hoc tests following Analysis of Variance (ANOVA), are displayed in Fig.~\ref{fig:ModelFidelity}-E4.
In Fig.~\ref{fig:ModelFidelity}-E4, the asterisks over the plots demonstrate the statistical significance level between the planners: $*$ indicates $p < 0.05$, $* *$ indicates $p < 0.01$, $* * *$ indicates $p < 0.001$, and $* * * *$ indicates $p < 0.0001$, where $p$ is the p-value reported by ANOVA. 
Non-significant differences are denoted with \textit{ns}.
Each model is optimized 100 times. 
Although Model~\#1 and Model~\#4 follow the same assumptions, Model~\#4 requires more time — 1.14 s longer — due to different mathematical expressions. 
Compared to Model~\#1, Model~\#2 and Model~\#3 make some simplifications, reducing solving time by 7 ms and 3 ms, respectively, but at the cost of fidelity.

In summary, the proposed vehicle model aligns well with real-vehicle data in open-loop simulations.	
Neglecting the longitudinal load transfer and friction circle limit can lead to an overestimation of vehicle mobility, while the reduction in solving time due to these simplifications is minimal.
The proposed model accounts for both load transfer and the friction circle limit, and it is more computationally efficient than the benchmark model in \cite{talbot2023shared}.

\section{Simulation \& Experimental Results}
\label{sec:Results}

\subsection{Race Track: High Performance Racing}
\label{sec:Racing}
\begin{figure*}
\centering
\ifthenelse{\boolean{PngBool}}{%
    \includegraphics[width=0.93\textwidth]{PngVer/Fig4RacingResult.png}
}{%
    \includegraphics[width=0.93\textwidth]{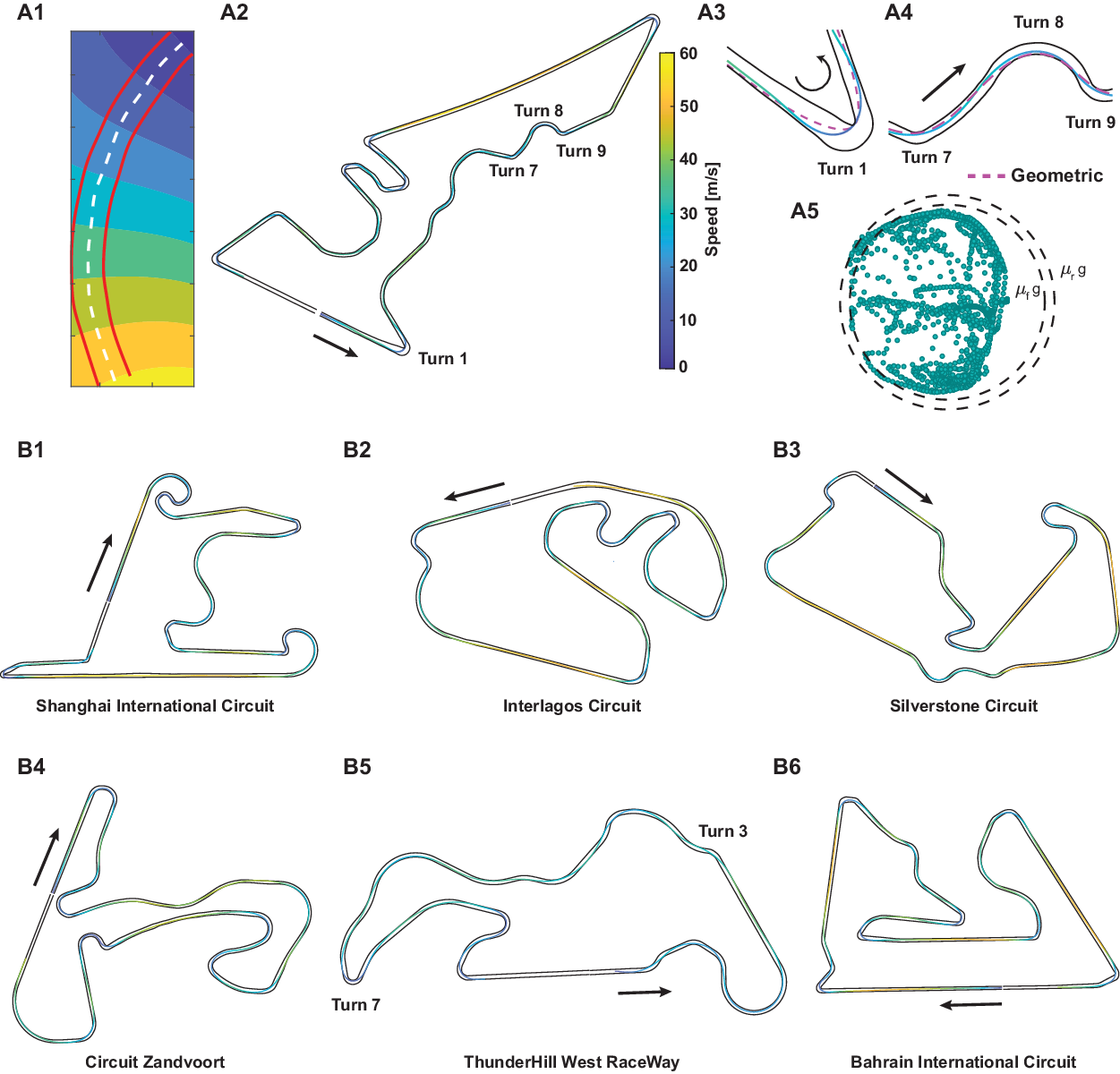}
}
\caption{The results in Sec.~\ref{sec:Racing}1. (\textbf{A1}) The contour for the cost-to-go. (\textbf{A2}) The trajectory of the vehicle in the Circuit of The Americas (COTA). The color on the line represents the speed of the vehicle. (\textbf{A3-A4}) The detailed plot of select turns in COTA. The dashed magenta line represents the geometric line that has the least curvature. (\textbf{A5}) C.G. acceleration plot. (\textbf{B1-B6}) The trajectories of the vehicle in Shanghai International Circuit, Interlagos Circuit, Silverstone Circuit, Circuit Zandvoort, Thunderhill West RaceWay, Bahrain International Circuit. The line colors follow the same color bar in subplot \textbf{A2}.}
\label{fig:RacingResult}
\end{figure*}

\begin{figure*}
\centering
\ifthenelse{\boolean{PngBool}}{%
    \includegraphics[width=0.9\textwidth]{PngVer/Fig6RealVehRace.png}
}{%
    \includegraphics[width=0.9\textwidth]{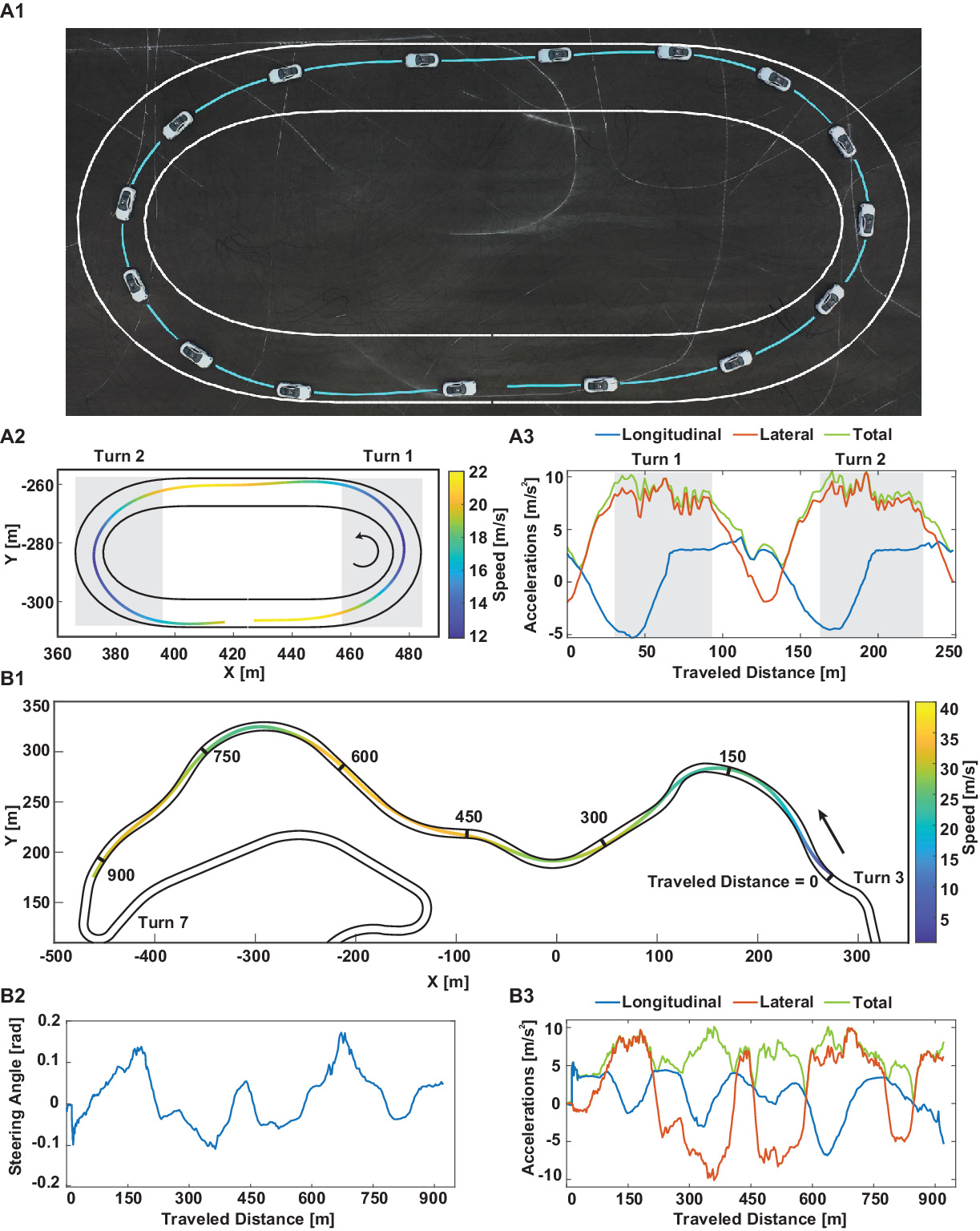}
}
\caption{The real vehicle testing results in Sec.~\ref{sec:Racing}2. (\textbf{A1}) An aerial view showing the path of oval racing. The vehicle poses are superpositioned from the test video with an interval of 1 s. The inner and outer white oval lines represent the actual boundaries of the oval race track. The cyan solid line represents the path of C.G. point.
(\textbf{A2}) The trajectory of the vehicle in the oval race track. The color on the line represents the speed of the vehicle. The rotation mark means that the vehicle is moving counter-clockwise. Two shaded gray rectangles represent the regions for Turn 1 and Turn 2.
(\textbf{A3}) The plot of longitudinal, lateral, and total accelerations along the traveled distance in the oval race track. The regions of Turn 1 and 2 are emphasized with shaded gray rectangles.
(\textbf{B1}) The trajectory of the vehicle in the Thunderhill West Raceway. The color on the line represents the speed of the vehicle. The inner and outer black solid lines represent the boundaries for the C.G. point of the vehicle. The traveled distances are labeled on the path with an interval of 150 m.
(\textbf{B2}) The plot of steering angle along the traveled distance in Thunderhill West Raceway.
(\textbf{B3}) The plot of longitudinal, lateral, and total accelerations along the traveled distance in Thunderhill West Raceway.
}
\label{fig:RealVehRace}
\end{figure*}

To test the proposed algorithm to its limits, racing applications are the optimal choice, as their primary goal is to safely exploit the vehicle's dynamic limits to achieve maximum speed on a racetrack.	
The boundaries on the racetrack naturally define the drivable region, thereby creating the envelope.	
% Unlike the state of the art \cite{dallas2023adaptive}, which utilizes a hierarchical planning and control architecture to divide the task, the proposed Envelope MPC only uses information about the vehicle states and the track boundaries to plan the trajectory.	
Simulation studies are conducted in Sec.~\ref{sec:racing_simulation_study}.
To further validate the proposed algorithm, two physical experiments are conducted on two tracks as described in Sec.~\ref{sec:racing_experimental_validation}.

For these racing scenarios, a cost-to-go term is added to the cost function to meet the objectives of racing in addition to using the standard formulation described in Sec. \ref{sec:env_mpc}. This cost-to-go term is described next.

\textbf{Cost to Go}: To motivate the MPC to drive the vehicle as fast as possible, a cost-to-go term is designed to promote maximization of the vehicle's travel distance in each prediction horizon. 
Assuming the vehicle's starting position ($x_0, y_0$) corresponds to the closest centerline point ($c_0$) at the arc length $s_0$ on the track, the farthest centerline point used to calculate the cost-to-go term corresponds to $s_f = s_0 + u_{x_{\max}} T_p$.
In Fig. \ref{fig:RacingResult}-A1, the bottom point corresponds to $s_0$, and the top right point at the end of the lane corresponds to $s_f$.
For each point on the centerline (corresponding to arc length $s_p$), a total of 15 points (including the centerline point itself) are uniformly sampled across the track width in the direction perpendicular to the centerline. 
All sampled points are labeled with the same arc length $s_p$.

For all points with arc length $s_p$, the cost-to-go terms are defined as ($s_f - s_p$).	
In this setting, the cost is highest at the starting point and reduces to zero at the end. 
At a given arc length, the vehicle incurs the same cost regardless of its lateral position within the track bounds.
To estimate the cost in optimization, a third-order polynomial regression is performed to express the cost with respect to its location.
Thus, the cost-to-go term is expressed as: 
\begin{equation}
    J_{\text{go}} = \sum_{i=0}^{3}\sum_{j=0}^{3-i} w_{\text{go}}(i, j)x^iy^j
\end{equation}
Fig. \ref{fig:RacingResult}-A1 shows the contour plot of the fitted cost. 
This cost is included as $J_{\text{specific}}$ in \eqref{eq:genericCostFunction} as a specific cost for racing applications; i.e., $J_{\text{specific}} = J_{\text{go}}$.

\subsubsection{Simulation Study}\label{sec:racing_simulation_study}

\begin{table}
\caption{Simulated Lap Time and Length of Spatial Envelope MPC}
\label{tab:track_data}
\centering
\begin{tabular}{lcc}
\hline
\textbf{Name} & \textbf{Lap Time (s)} & \textbf{Length (m)} \\
\hline
Bahrain & 141.96 & 5200 \\
COTA & 153.36 & 5200 \\
Catalunya & 132.56 & 4500 \\
Interlagos & 114.66 & 4100 \\
Shanghai & 149.76 & 5250 \\
SilverStone & 144.66 & 5700 \\
Thunderhill 2-mile & 83.36 & 2450 \\
Zandvoort & 124.86 & 4100 \\
\hline
\end{tabular}
\end{table}

In Fig. \ref{fig:RacingResult}-A2, the algorithm is put to test in Circuit of The Americas (COTA), a well-known race track that has been in Formula 1  Grand Prix since 2012. 
The trajectory of the vehicle is plotted and the color gradient on the trajectory represents the vehicle speed. 
The vehicle is capable of accelerating on the straight runs and decelerating before the turns so that it does not lose control. 

% \Tulga{No information is given about the simulation parameters.}\Siyuan{Discuss here}

Several key turns are magnified in Fig. \ref{fig:RacingResult}-A3 and A4 to demonstrate the performance of the Spatial Envelope MPC.  In Fig. \ref{fig:RacingResult}-A3, a hairpin turn scenario is depicted, in which the vehicle must execute a sequence of dynamic maneuvers: first, applying heavy braking to decelerate, followed by aggressive steering to make the sharp turn, and finally accelerating upon exiting the turn.
The trajectory plots in this figure indeed show this desired behavior.
The geometric reference line, shown in purple, is derived from \cite{braghin2008race}, where each segment is modeled using third-order polynomials, formulated to solve an optimization problem that minimizes curvature along the path.
A notable deviation from the least-curvature geometric line is observed in the vehicle’s trajectory, indicating a distinct driving behavior that emerges during the hairpin turn.
Compared to the least-curvature geometric line, the vehicle exhibits a late braking pattern when approaching the turn and reaches the apex later. 
This strategy enables the vehicle to maintain higher exit speeds, resulting in faster acceleration after the turn compared to following the geometric line.
This driving technique, commonly referred to as the late-apex strategy, is frequently employed by professional drivers to achieve faster lap times.
In this setting, the algorithm is not assisted by prior knowledge but designs the trajectory on the fly, naturally embedding the late-apex strategy within the constraints and cost formulation.

In Fig. \ref{fig:RacingResult}-A4, the vehicle does not follow the geometric line in relatively high-speed turns, either. 
After Turn 7, the vehicle intentionally deviates from the apex to slow down further and execute a more aggressive turn, eventually hitting the late-apex point.	
The rationale is that, after Turn 9, the track transitions into a long straight run where maintaining high initial speed is crucial for the vehicle.	
Adhering strictly to the geometric line would cause loss of speed during the turns, which is undesirable.	
The Spatial Envelope MPC predicts 6.75 s into the future for 6.75 s when making decisions to prevent settling on a locally optimal trajectory.

In Fig. \ref{fig:RacingResult}-A5, the longitudinal and lateral accelerations are sampled throughout the entire lap and are compared to the front and rear friction force limits shown as dashed circles. 
The vehicle is safely pushed to the limits and is often operating near the limits in this application.

The same formulation is taken without any re-tuning and applied to other world-famous racetracks, as well; namely, Shanghai International Circuit (Fig.~\ref{fig:RacingResult}-B1), Interlagos Circuit (Fig.~\ref{fig:RacingResult}-B2), Silverstone Circuit (Fig.~\ref{fig:RacingResult}-B3), Circuit Zandvoort (Fig.~\ref{fig:RacingResult}-B4), Thunderhill West Raceway (Fig.~\ref{fig:RacingResult}-B5), and Bahrain International Circuit (Fig.~\ref{fig:RacingResult}-B6).
The algorithm can safely operate the vehicle across all six racetracks without modifications to the formulation; the only changing parameter is the boundary of the racetrack.
% \Tulga{We should report all lap times. They can serve as benchmarks for all future papers that want to beat this formulation.}
The travel length and lap time for each comparison is reported in Table. \ref{tab:track_data}.
These results highlight the method's general applicability and zero-shot capability.

\subsubsection{Experimental Validation}\label{sec:racing_experimental_validation}

To evaluate its real-world performance, the Spatial Envelope MPC algorithm is deployed on a prototype vehicle based on a 2019 LC 500 from Toyota Research Institute, as illustrated in Fig.\ref{fig:ModelFidelity}-B. 
The algorithm is tested on an oval race track and the Thunderhill West Raceway. 

The oval race track was virtually defined within one of the skid pads in Thunderhill Raceway Park. 
Fig. \ref{fig:RealVehRace}-A1 illustrates the sequence of vehicle poses at 1-second intervals, demonstrating the algorithm’s ability to maintain the vehicle within track boundaries.
To further analyze the vehicle's mobility, the speed and acceleration profiles are presented in Fig.\ref{fig:RealVehRace}-A2 and A3, respectively.
In Fig. ~\ref{fig:RealVehRace}-A2, a maximum speed of 21.03 m/s is achieved near the entrance of Turns 1 and 2, while the minimum speed of 11.88 m/s occurs near the apex.
Fig.~\ref{fig:RealVehRace}-A3 shows that the maximum total acceleration of 10.57 m/s² is reached in Turns 1 and 2, constrained by the maximum available friction forces.
During Turns 1 and 2, the total acceleration remains close to its maximum, indicating that the algorithm effectively drives the vehicle to its dynamic limits. During the two turns, the majority of the available friction forces are utilized as centrifugal forces to sustain curvilinear motion.

Fig.~\ref{fig:RealVehRace}-B1 to B3 present the results of testing the Spatial Envelope MPC algorithm on the Thunderhill West Raceway, a significantly more complex and dynamic track compared to the virtual oval track.
In Fig.~\ref{fig:RealVehRace}-B1, the vehicle's trajectory is visualized along the track, with speed color gradients along the traveled path. 
The traveled distance is marked as vertical black lines on the racetrack at 150 m intervals.
The vehicle starts from Turn 3 and navigates through the track, eventually reaching Turn 7 (marked in  Fig.~\ref{fig:RacingResult}-B5), demonstrating the algorithm’s ability to adapt to varying curvatures and track geometries in real time. 
This specific section is selected because the racetrack as a whole includes challenging terrain topologies, which could significantly impact the performance.
% adversely affect the performance of the envelope MPC algorithm.
because the current study has not yet considered the effects of terrain topologies on the racetrack. 
However, the section between Turn 3 and Turn 7, while providing significant challenges, is relatively flat, making it an ideal segment for testing.
Due to heavy rain prior to the test date, the parameters for the front and rear friction coefficients are defined more conservatively as $(\mu_f, \mu_r) = (0.9, 0.95)$.

Notably, the speed profile along the trajectory shows that the algorithm effectively manages the vehicle’s velocity depending on the curvature of the road. 
The vehicle reaches higher speeds on straight runs, with a maximum of 35.69 m/s (79.86 mph / 128.52 kph) recorded along a high-speed turn section  between the 450 m and 600 m markers. 
The time-to-goal is 35.52 s for a traveled distance of 921 m. 

Fig. ~\ref{fig:RealVehRace}-B2 illustrates the vehicle's steering trajectory along the race track, where aggressive maneuvers are evident from the significant steering angles executed at high speeds.
Fig. ~\ref{fig:RealVehRace}-B3 shows the vehicle's longitudinal, lateral, and total accelerations over the traveled distance. 
The total acceleration peaks at around 10 m/s², indicating that the algorithm operates the vehicle near its dynamic limits, especially when negotiating tight corners.
The lateral acceleration component dominates during cornering, confirming again that the majority of available friction forces are used to counteract centrifugal forces and maintain traction in the turns.
Meanwhile, longitudinal acceleration spikes primarily occur during straight sections when the vehicle accelerates or decelerates before entering a turn.
This careful balance between longitudinal and lateral forces showcases the algorithm’s ability to push the vehicle to its performance limits while maintaining stability and control throughout the track.

\subsection{On-Road: Emergency Collision Avoidance} \label{sec:CIS}

\begin{figure}
\centering
\includegraphics[width=0.45\textwidth]{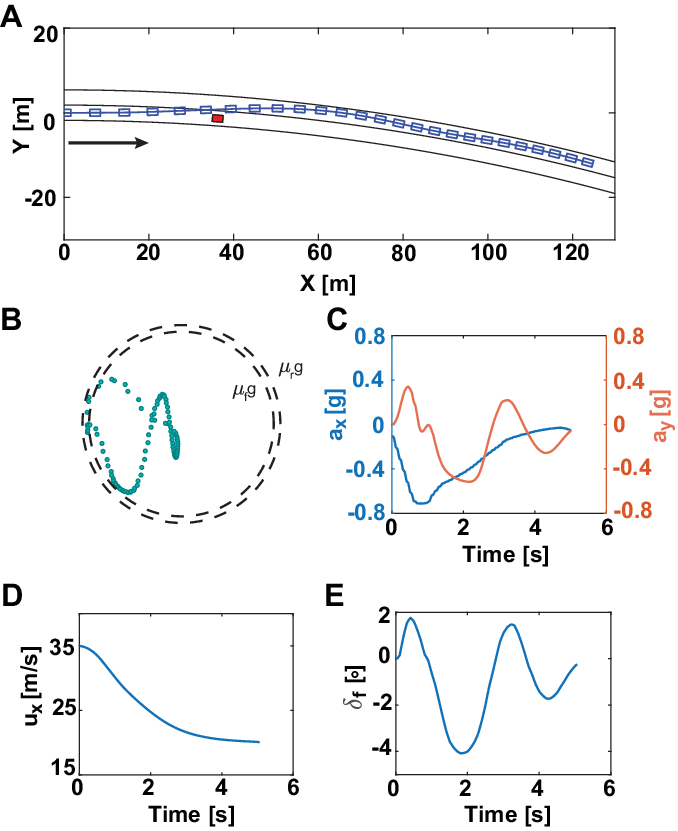}
\caption{The results in Sec.~\ref{sec:CIS}. (\textbf{A}) The plot of the X-Y path. (\textbf{B}) The plot for longitudinal and lateral accelerations. (\textbf{C}) The time histories of longitudinal and lateral accelerations. The longitudinal acceleration is labeled on the left and the lateral acceleration is labeled on the right. In subplots \textbf{B} \& \textbf{C}, the gravitational constant $g$ is used as units for both longitudinal and lateral accelerations. (\textbf{D}) The time history for the longitudinal speed.  (\textbf{E}) The time history for the steering angle.}
\label{fig:CIS}
\end{figure}

% Describe scenario
A similar spatial envelope constraint concept can also be applied to emergency collision avoidance in on-road scenarios.	
This section demonstrates this in simulation.
As shown in Fig. \ref{fig:CIS}-A, the vehicle (depicted in blue) is traveling through a curved highway at an initial speed of 35 m/s (77 mph).
At $t = 0$, a red vehicle that is 35 meters ahead suddenly appears, blocking the original lane and remaining stationary.
This situation constitutes an emergency because the vehicle has only one second to change lanes completely.	
The controller is also programmed to reduce the vehicle's speed to ensure it can be safely handed back to the human operator after the emergency maneuver.	
Compared to \cite{wurts2021collision} with constant speed settings, this requirement brings additional challenges to coordinating the longitudinal and lateral tire forces. 
Intuitively, the benefits of these settings are twofold: 1) The additional degree of control freedom allows for better maneuvering capability, enabling the vehicle to perform more aggressive maneuvers and thereby enhancing safety in emergency scenarios.
2) It provides human drivers with a better takeover experience when the vehicle is stabilized in the other lane.

\textbf{Speed Cost}: To accomplish this, a simple cost function is added in $J_{\text{specific}}$ to actively slow down the vehicle:

\begin{equation}
    J_{\text{speed}} = \int_{t_0}^{t_f} (u_x - u_\text{des})^2
\end{equation}

Here the $u_\text{des}$ is 20 $m/s$ (44 mph) to ensure that the human driver can control the vehicle after the emergency maneuver.

The envelope used in this formulation is shown in Fig. ~\ref{fig:res_shape}-A3.

% Describe formulation

% Describe Graph 
The resultant path of the collision imminent steering scenario is shown in Fig.~\ref{fig:CIS}-A. 
Vehicle positions are represented by a sequence of blue rectangles, each separated by a 0.2-second interval.	
The solid blue curve connecting each blue rectangle represents the history of the vehicle's C.G. point. 
The histories of C.G. accelerations are plotted in Fig.~\ref{fig:CIS}-B and C. 
In Fig.~\ref{fig:CIS}-B, the x-axis represents longitudinal acceleration, while the y-axis denotes lateral acceleration.
Inner dashed circle indicates the acceleration limit of $\mu_f g$, while the outer dashed circle indicates the limit of $\mu_r g$.	
Several points in Fig.~\ref{fig:CIS}-B are located between the two dashed circles, indicating that the algorithm pushes the vehicle to its dynamic limits to avoid the obstacle.
The time histories of two accelerations are shown in Fig.~\ref{fig:CIS}-C. To avoid the obstacle, the vehicle initiates a large steering angle, resulting in large lateral acceleration and the peak appears at around 0.5 s. 
An interesting trade-off between longitudinal and lateral acceleration is observed as the vehicle stabilizes in the second lane at approximately $t = 2 \text{ s}$.	
Given the speed cost, the vehicle should decelerate as rapidly as possible. However, the algorithm opts to apply less force for deceleration and instead actively redistributes forces in the lateral direction to assist with turning.	
The time history of longitudinal speed is shown in Fig.~\ref{fig:CIS}-D. 
The speed smoothly decreases from 35 m/s to 20 m/s which is the desired exit speed defined for this scenario. 
The steering angle trajectory is shown in Fig.~\ref{fig:CIS}-E. 
As mentioned before, the first peak of the steering angle appears at around 0.5 s, corresponding to the first peak of lateral acceleration. 
This maneuver leads to a large lateral speed and yaw rate to make sure that the vehicle can avoid the obstacle in time. 
Consequently, greater counter-steering is necessary to stabilize the vehicle and keep it within the lane boundary, which is shown as the first valley (at around 2 s) in Fig.~\ref{fig:CIS}-E.

In summary, the Spatial Envelope MPC is capable of performing collision avoidance in emergency scenarios on road. 
This method is capable of applying aggressive maneuvers to avoid obstacles while preserving vehicle stability and ensuring safety.	 
Experimental results confirm that the vehicle adheres to dynamic limits, achieves the desired speed reduction, and executes safe lane-changing maneuvers, thereby demonstrating the efficacy of the proposed method.

\subsection{Off-Road: Trail Navigation}
\label{sec:offroad}
\begin{figure*}
\centering
\ifthenelse{\boolean{PngBool}}{%
    \includegraphics[width=0.8\textwidth]{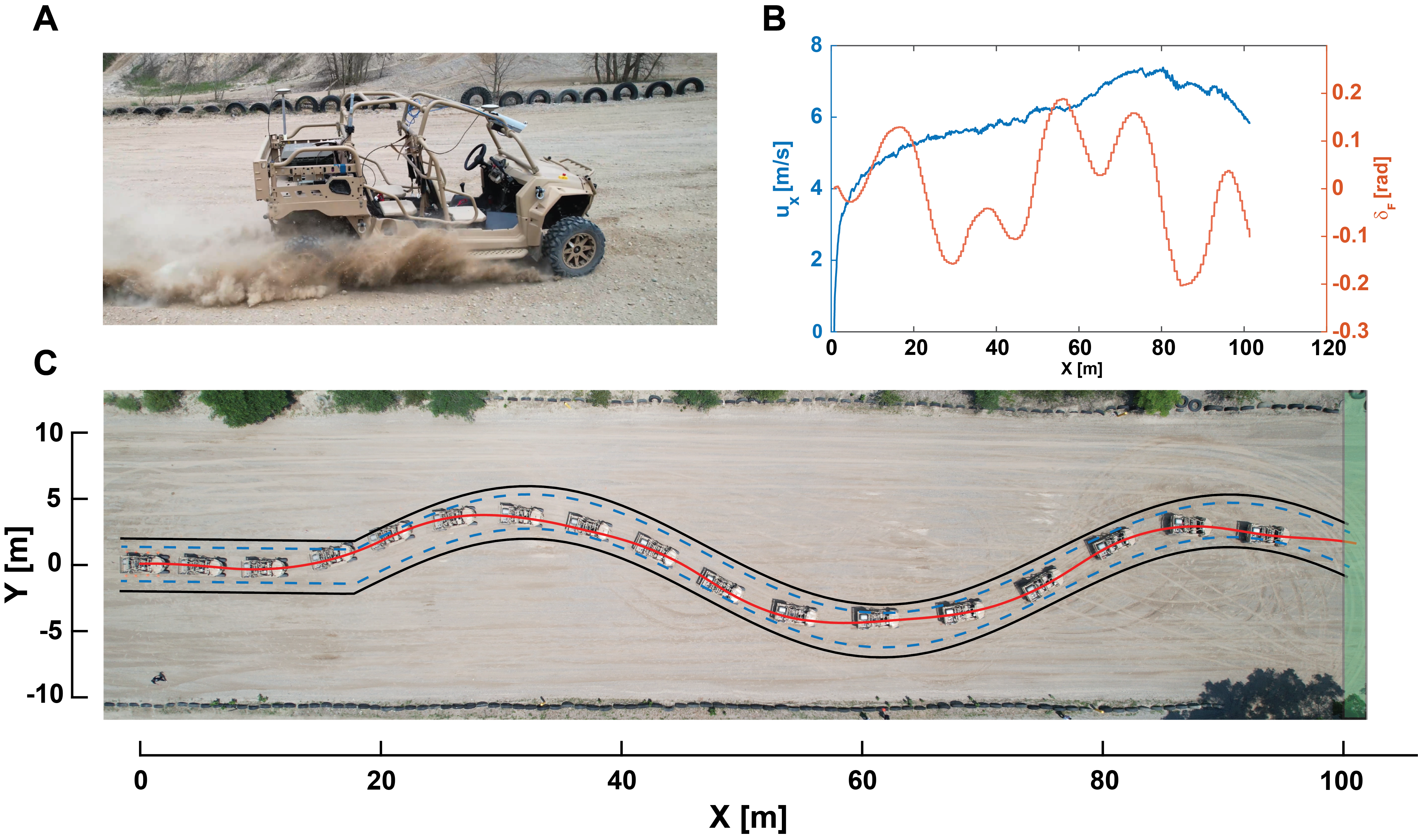}
}{%
    \includegraphics[width=0.8\textwidth]{Fig6OffRoad.png}
}
\caption{The results in Sec.~\ref{sec:offroad}. (\textbf{A}) A snapshot of the test vehicle - Polaris MRZR - driving in the off-road environment. (\textbf{B}) The plots of longitudinal speed and steering angle versus the X position. (\textbf{C}) An aerial view showing the X-Y path from one experiment. The vehicle poses are superpositioned from the aerial video. The upper and lower black solid lines represent the exact trail boundaries. The blue dashed lines represent the boundaries for the center of gravity (C.G) of the vehicle. The red solid line represents the X-Y path of the vehicle's C.G position. The small green area at the end denotes the ending region of the vehicle. }
\label{fig:OffRoad}
\end{figure*}

To further demonstrate the real-world implementation capability and general usage of the algorithm, the Spatial Envelope MPC was installed on an off-road vehicle, as shown in Fig. \ref{fig:OffRoad}-A. 
The vehicle running autonomously here is a gasoline-powered four-seater Polaris MRZR.
As described in Sec. \ref{sec:VD}, the vehicle model proposed in this study is primarily designed for on-road purposes. However, as indicated in \cite{yu2024Real}, the on-road tire model fails to accurately capture tire-terrain interactions, leading to prediction errors when applied to off-road conditions.
To address these challenges, this section utilizes the vehicle model introduced in \cite{yu2024Real, dallas2021terrain}, which accounts for deformable terrain to ensure safe operation in off-road environments.
As illustrated in Fig. \ref{fig:OffRoad}-C, the scenario involves the vehicle navigating through a predefined trail without crossing the boundaries. 
This setup tests the algorithm's ability to handle real-world off-road navigation challenges.
This trail scenario is commonly encountered in off-road navigation and rally racing, making it a relevant and practical test case for the algorithm.
In this context, the two black lines represent the boundaries of the actual vehicle chassis, while the blue dashed lines indicate the derived boundaries for the vehicle's center of gravity (C.G.).
As long as the vehicle's C.G. does not intersect the blue dashed lines, the vehicle is considered to be safely within the trail boundaries.
The small green area at the end of the trail represents the vehicle's target position. The vehicle is considered to have reached the goal once its x-coordinate exceeds 100 m.	
For this scenario, the original formulation in Sec. \ref{sec:env_mpc} is modified in the cost definition as follows.

\textbf{Terminal Cost}: The terminal cost defined with $J_{\text{specific}}$ is written as: 
\begin{equation}
    \label{eq:off_road_terminal}
    J_{\text{specific}} = \frac{x_{f} - x_{\text{goal}}}{x_0 - x_{\text{goal}}}
\end{equation}
where $x_0$ and $x_f$ are the x position at the start and end point of the prediction horizon, respectively, and $x_\text{goal}$ is the x position of the goal region. 
This cost motivates the MPC to plan trajectories leading the vehicle to the goal region as quickly as possible. 
No prescribed speed is used in this formulation; the algorithm simultaneously determines the speed and steering command for the vehicle.

Fig. \ref{fig:OffRoad}-B shows the steering and speed trajectories of the vehicle and Fig. \ref{fig:OffRoad}-C displays the path plot of the vehicle's C.G. along with superimposed bird's-eye view images of the vehicle from a stationary drone footage.

The results demonstrate that the proposed method effectively identifies trajectories that stay safely within the designated boundaries.	
Instead of following the center of the track, the algorithm autonomously chooses to travel closer to the boundaries, resulting in a reduced path curvature.	
This approach provides the advantages of increased speed and enhanced time-to-goal performance.	
This improved performance is also apparent in the steering plot.	
Although the trail is modeled as a sinusoidal wave, the steering trajectory does not strictly adhere to this pattern. 
Variations in the steering angle enable the vehicle to achieve maximum speed along the trail.

In summary, the Spatial Envelope MPC demonstrates the capability to operate the vehicle in real-time within an off-road environment. 
Instead of adhering to a predefined trajectory, this method offers a high-performance solution that enhances off-road mobility within the trail.

\subsection{Computational Performance of Spatial Envelope MPC}

\begin{table}
\caption{Computational Performance of Spatial Envelope MPC}
\centering
\label{tab:Computational}
\begin{tabular}{c c c c}
\hline
  Scenario  &  \makecell{Solve Time \\ Statistics (ms)}  &   Scenario  &  \makecell{Solve Time \\ Statistics (ms)}\\
\hline

  % \makecell{Sim Racing\\ (\textbf{COTA})}  & 37.5$\pm$ 12.8  & \makecell{Sim Racing \\ (\textbf{Shanghai})}  & 36.6 $\pm$ 15.1 \\ \vspace{0.1em}
  % \makecell{Sim Racing\\ (\textbf{Interlagos})}  & 35.8 $\pm$ 12.9 &    \makecell{Sim Racing \\ (\textbf{Silverstone})} & 36.4 $\pm$ 13.4 \\ \vspace{0.1em}
  %  \makecell{Sim Racing \\ (\textbf{Zandvoort})} & 38.2 $\pm$ 16.6 & \makecell{Sim Racing \\ (\textbf{ThuderHill})} & 31.7 $\pm$ 13.6 \\ \vspace{0.1em}
  % \makecell{Sim Racing \\ (\textbf{Bahrain})} & 35.9 $\pm$ 23.1 &  & \\ \vspace{0.1em}
  %   \makecell{Physical Racing \\ (\textbf{Oval})}& 63.9 $\pm$ 20.4 &\makecell{Physical Racing \\ (\textbf{Thunderhill})}& 71.4 $\pm$ 25.9 \\ 
  %   \vspace{0.1em}
  %    \makecell{Emergent \\ Collision Avoidance} & 30.5 $\pm$ 10.8 & Off-road & 20.9 $\pm$ 9.8 \\

  \makecell{Sim Racing\\ (\textbf{COTA})}  & 37.5$\pm$ 12.8  
  & 
  \makecell{Sim Racing \\ (\textbf{Bahrain})} & 35.9 $\pm$ 23.1 
  \\ 
  
  \makecell{Sim Racing \\ (\textbf{Shanghai})}  & 36.6 $\pm$ 15.1
  &
  \makecell{Exp Racing \\ (\textbf{Oval})}& 63.9 $\pm$ 20.4 
  \\

  \makecell{Sim Racing\\ (\textbf{Interlagos})}  & 35.8 $\pm$ 12.9
  & 
  \makecell{Exp Racing \\ (\textbf{Thunderhill})}& 71.4 $\pm$ 25.9 
  \\ 
  
  \makecell{Sim Racing \\ (\textbf{Silverstone})} & 36.4 $\pm$ 13.4 
  &
  \makecell{Sim Emergency \\ Collision Avoidance} & 30.5 $\pm$ 10.8 
  \\ 
   
  \makecell{Sim Racing \\ (\textbf{Zandvoort})} & 38.2 $\pm$ 16.6 
  &
  \makecell{Exp Off-road \\ Navigation} & 20.9 $\pm$ 9.8
  \\ 

  \makecell{Sim Racing \\ (\textbf{Thunderhill})} & 31.7 $\pm$ 13.6 
  \\

   %   &  \\ \vspace{0.1em}
   % Off-road & 20.9 $\pm$ 9.8 &  & \\
      % \makecell{Path Tracking \\ (\textbf{Proposed})} & 28.0 $\pm$ 13.1 &  \makecell{Path Tracking \\ (\textbf{Benchmark})} & 12.2 $\pm$ 3.3 \\
      % \makecell{Path Tracking \\ (\textbf{Oval})} & 13.8 $\pm$ 9.1 & & \\
   \hline
% \hline
\end{tabular}
\end{table}

Table \ref{tab:Computational} summarizes the computational performance of each scenario in terms of mean and standard deviation of solve times. 
% The first value in the solve time statistics represents the mean, while the second value represents the standard deviation.	
All simulated scenarios were computed on an Alienware R15 desktop with an Intel i9-14900K CPU clocking at 3.2 GHz, and 64 GB of memory, with optimizations performed in a single thread.	
The algorithm generated optimal control commands in approximately 35-40 ms, which is significantly below the 100 ms target threshold.
For the physical tests in racing applications, a ThinkPad laptop with a Xeon(R) W-11855M CPU, clocking at 3.2 GHz, and 32 GB of memory was used.
However, the computational performance was limited due to the laptop running on battery power.
Therefore, the average solve time for each iteration increased to 60-70 ms, but still met the 100 ms threshold on average.
For those instances where the optimization times reached 100 ms, the optimizations were terminated after 100 ms. 
If no solution was found within this time, the MPC switched to the new initial point and initiated a new calculation while still executing the optimal commands from the last successful solution.
Consequently, the maximum optimization time was capped at 100 ms.
Such overrun of the 100 ms limit happened in 64 OCPs out of a total of 1158 across the Oval and Thunderhill experiments.
For the off-road experiment, a Dell Precision 5500 Workstation laptop with an Intel Core i7-10850H processor, clocked at 2.7 GHz, and 62.6 GB of memory was used to compute the vehicle's trajectories.
% \Tulga{Table shows solve times of 21 ms. Why do we make it sound like we hardly met the 100 ms target?}\Siyuan{Because in Real Racing Applications, the solve time is close to 100ms due to the computational power of James' Computer }
These computational performances indicate that the algorithm can operate at 10 Hz for all scenarios on modern CPUs.

% \Tulga{Does this paragraph refer to tracking case study? If so, it should be removed.}
% A significant performance difference is observed at the bottom of the table, with the benchmark method demonstrating more than twice the speed of the proposed method.	
% However, since the benchmark method relies on a known analytical expression, its computational performance depends on the complexity of this expression and could deteriorate as the complexity increases.	
% In contrast, the proposed method relies solely on the locations of reference points, making it applicable to non-analytic references, as well.
% This behavior further proves the scalability of this method in path tracking algorithm. 
% \Tulga{We did not prove any scalability. We did not even demonstrate it. I would remove this claim.}\Siyuan{Let's discuss this, the benchmark method suffers when the analytic expression becomes increasingly complicated, but ours directly uses the points, instead of expressions so it does not matter what kind of trajectory that we want to track.}

\subsection{Spatial Envelope Planning}\label{sec:env_planning_result}

In Fig. \ref{fig:res_shape}, the planned spatial envelopes of three design schemes are compared: pure reinforcement learning, optimization, and reinforcement learning guided optimization. In RL guided optimization, parameters $L, W$ from RL are scaled by 0.8 to ensure proper initialization. The computation time is compared in Fig. \ref{fig:res_time}. 

The reinforcement learning approach exhibits the fastest computation, but the planned blocks are not necessarily optimal, and the constraints are not guaranteed to be satisfied. 
Although solely using optimization produces near-optimal results within the constraints, the algorithm is prone to be trapped in local minimums, and has the highest computation time. 
Using reinforcement learning trained policy to guide optimization reduces the risk of local minima, enforces constraint satisfaction, and reduces computational time. 
This result demonstrates the reliability and effectiveness of the reinforcement learning guided optimization scheme for envelope design. 

\begin{figure}
    \centering
    \includegraphics[width=0.48\textwidth]{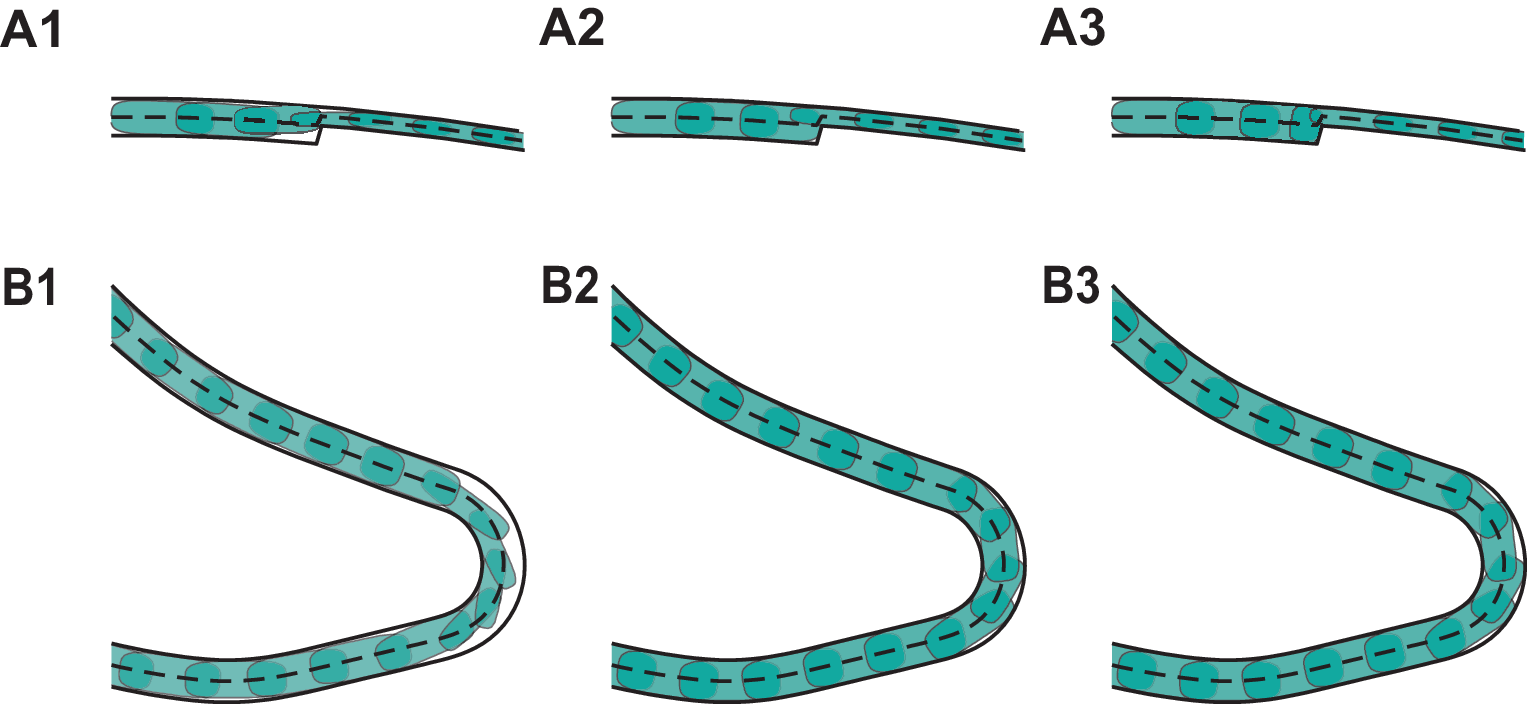}
    \caption{(\textbf{A}) A segment of the designed spatial envelopes for the Collision Imminent Steering scenario. (\textbf{B}) Spatial envelopes designed by the three methods for Thunderhill West race track. (\textbf{1}) RL policy; (\textbf{2}) Optimization only; (\textbf{3}) RL guided optimization.} 
    \label{fig:res_shape}
\end{figure}

\begin{figure}
    \centering
    \includegraphics[width=0.4\textwidth]{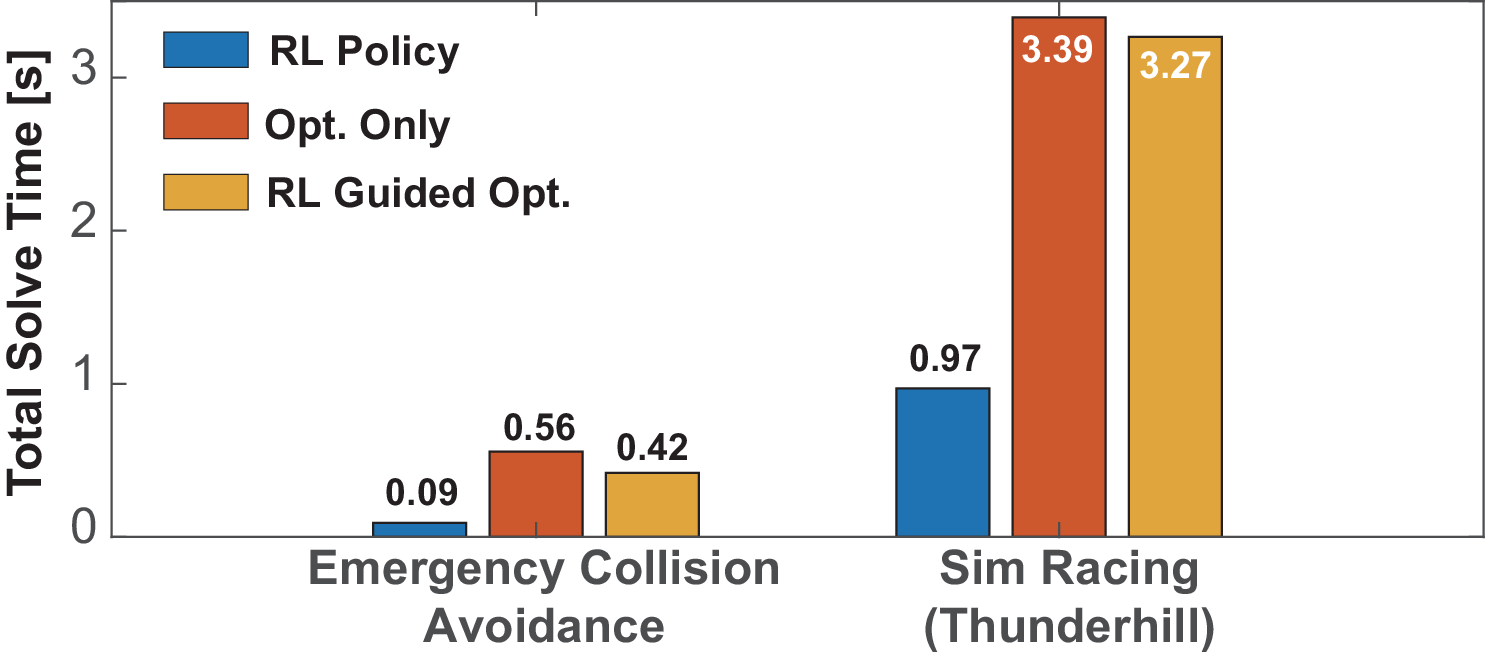}
    \caption{Comparison of computational times of the three spatial envelope design schemes for Emergency Collision Avoidance and Sim Racing (Thunderhill) scenarios.}
    \label{fig:res_time}
\end{figure}

\section{Conclusion}\label{sec:conclusion}

This study introduces a spatial envelope based model predictive control framework designed to push a vehicle to its limits autonomously, safely, and in real time.	
A new 3 DoF single-track model is introduced.
This model is specifically tailored for closed-loop optimization algorithms and captures efficiently the key vehicle dynamics in high-performance driving scenarios.	
A new mathematical formulation for spatial envelope constraints is developed to plan dynamically feasible and collision-free trajectories while maintaining scalability for complex scenarios. 
This formulation explicitly defines the entire drivable region, in contrast to the state of the art that constrains each collocation point to a specified track segment. 
As a result, the new envelope formulation removes the need for a restriction on longitudinal maneuvers.
This novelty enables, for the first time, one-level speed-varying trajectory planning and control within spatial envelope MPC and eliminates the need for a reference.
This work also presents a technique to design the envelope effectively.

% The vehicle dynamics model is validated by comparing it with real-world driving data and is found to be a more accurate and computationally efficient prediction model compared to the benchmark models.
% The Envelope MPC is validated in various high-performance driving scenarios and tested in racing applications in both simulated and physical experiments to demonstrate its high performance and zero-shot capabilities.	
% An emergent collision avoidance scenario is designed to test this method's application in emergent on-road situations, highlighting its ability to coordinate longitudinal and lateral tire forces.	
% The algorithm is further transferred and tested in an off-road vehicle to navigate through trails, showcasing its applicability outside on-road scenarios.
% This work also presents a method for effective envelope design.
% In light of the presented results, these original contributions are recognized as a significant advancement toward enhancing the mobility of autonomous vehicles in high-performance scenarios.	

The vehicle dynamics model is first validated against real-world driving data, demonstrating improved accuracy and computational efficiency compared to benchmark models. 
Building on this foundation, the Spatial Envelope MPC is evaluated across a variety of high-performance driving scenarios, including both simulated and physical racing experiments, to establish its strong performance and zero-shot generalization capabilities.
To further assess its robustness, an emergency collision avoidance scenario is designed, highlighting the controller’s ability to coordinate longitudinal and lateral tire forces in critical on-road situations.
The algorithm is then transferred to an off-road vehicle and tested on unstructured trails, illustrating its applicability beyond traditional on-road settings.
In parallel, a method for effective spatial envelope design is introduced to support constraint formulation within the planning framework. Taken together, these contributions represent a significant step forward in enhancing the mobility and versatility of autonomous vehicles operating in high-performance environments in a reference-free setting.

Assessing the robustness of the algorithm is identified as the next step, as the current work assumes perfect knowledge of the environment.	
Although the proposed Spatial Envelope MPC is motivated by and evaluated in performance driving scenarios, the same methodology could be applied to other robotic systems such as drones, wheeled robots, and robotic arms.
Applications to other problem domains are identified as important future research directions.

\bibliographystyle{IEEEtran}
\bibliography{IEEEfull}    

\begin{IEEEbiography}[{\includegraphics[width=1in,height=1.25in,clip,keepaspectratio]{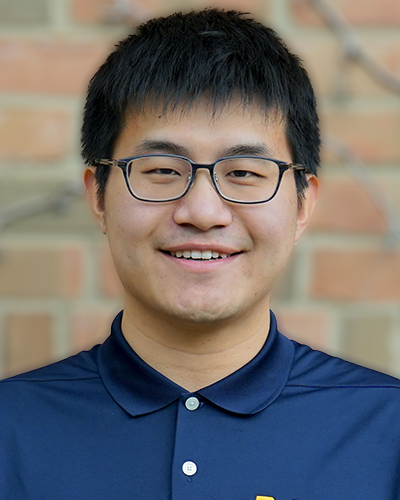}}]{Siyuan Yu}
received his B.S.E in Electrical and Computer Engineering from Shanghai Jiao Tong University, China, in 2020 and the B.S.E and M.S.E in Mechanical Engineering from University of Michigan, Ann Arbor in 2020 and 2022, respectively. He is currently pursuing the Ph.D. degree in Mechanical Engineering at the University of Michigan, Ann Arbor. His research interests include modeling, system identification, motion planning and control, with respect to vehicle systems.
\end{IEEEbiography}

\begin{IEEEbiography}[{\includegraphics[width=1in,height=1.25in,clip,keepaspectratio]{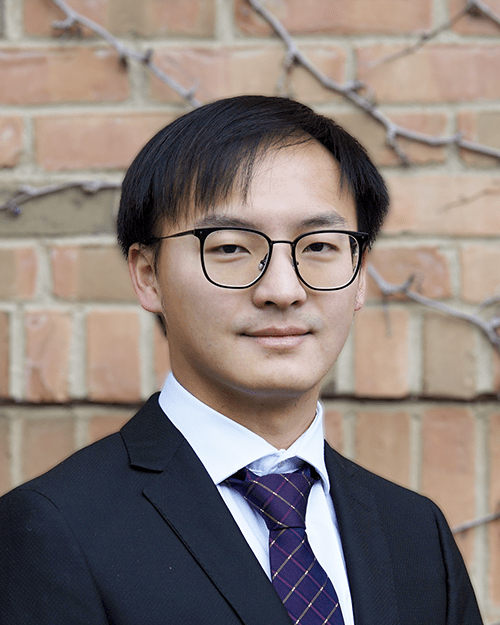}}]{Congkai Shen}
 earned his Bachelor of Science in Electrical and Computer Engineering from Shanghai Jiao Tong University, China, in 2020. Subsequently, he completed his Bachelor of Science and Master of Science in Mechanical Engineering at the University of Michigan, Ann Arbor, in 2020 and 2022, respectively. Currently, he is pursuing the Ph.D. degree in Mechanical Engineering at the University of Michigan, Ann Arbor. His research focuses on modeling, system identification, motion planning, and control, in the context of vehicle systems.
\end{IEEEbiography}

\begin{IEEEbiography}[{\includegraphics[width=1in,height=1.25in,clip,keepaspectratio]{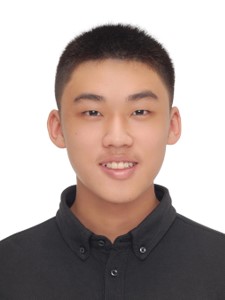}}]{Yufei Xi}
is currently pursuing B.S.E in Mechanical Engineering from Shanghai Jiao Tong University, and B.S.E in Computer Engineering from University of Michigan, Ann Arbor. His research interests include modeling, motion planning and control of dynamic systems. 
\end{IEEEbiography}

\begin{IEEEbiography}[{\includegraphics[width=1in,height=1.25in,clip,keepaspectratio]{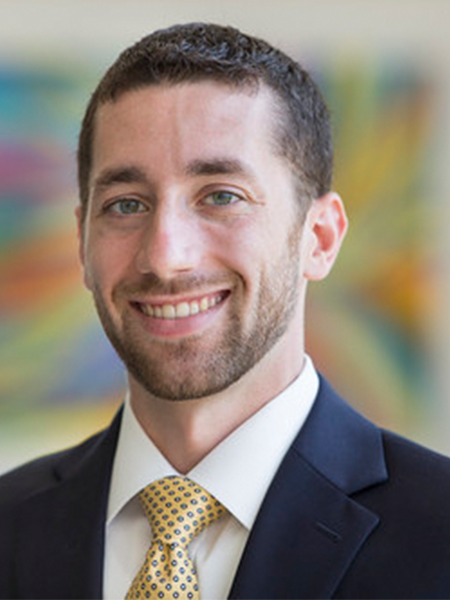}}]{James Dallas}
received the B.S. degree from the Pennsylvania State University, University Park, PA USA, and the M.S. and Ph.D. degrees from the University of Michigan, Ann Arbor, MI USA, in 2017, 2018, and 2021, respectively, all in Mechanical Engineering. He is currently a Research Scientist with the Toyota Research Institute. His research interests include modeling, system identification, system dynamics and control, shared control, and adaptive and robust optimal control, with applications to vehicle systems.
\end{IEEEbiography}

\begin{IEEEbiography}[{\includegraphics[width=1in,height=1.25in,clip,keepaspectratio]{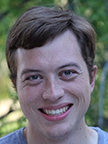}}]{Michael Thompson}
received a B.S. degree in aerospace engineering from the University of Notre Dame and a M.S. degree in aerospace engineering from Stanford University. He is currently a member of the Extreme Performance Intelligent Control group at Toyota Research Institute. His research interests include vehicle dynamics modeling and control algorithms for high performance autonomous vehicles
\end{IEEEbiography}

\begin{IEEEbiography}[{\includegraphics[width=1in,height=1.25in,clip,keepaspectratio]{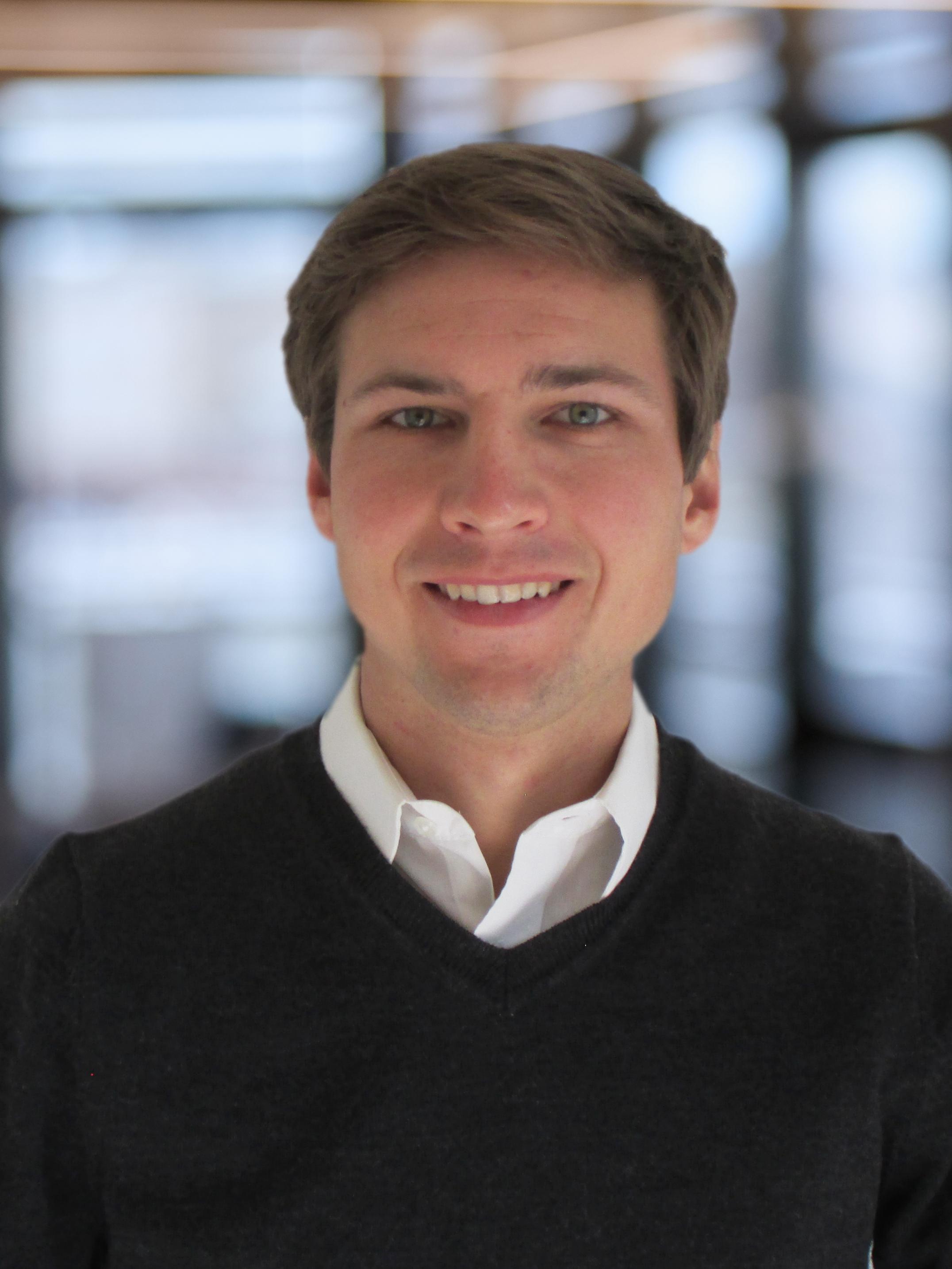}}]{John Subosits}
received his B.S.E. degree in mechanical and aerospace engineering from Princeton University and M.S. and Ph.D. degrees in mechanical engineering from Stanford University.  Currently, he leads the Extreme Performance Intelligent Control group at Toyota Research Institute (TRI).  His research interests include algorithms for vehicle control that match the performance, robustness, and adaptability of the best human (racing) drivers.
\end{IEEEbiography}

\begin{IEEEbiography}[{\includegraphics[width=1in,height=1.25in,clip,keepaspectratio]{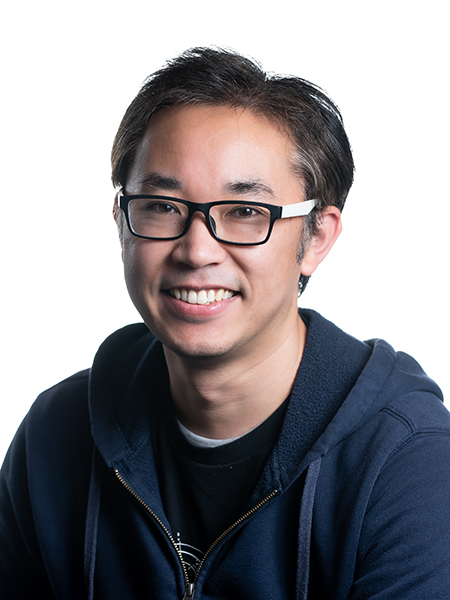}}]{Hiroshi Yasuda}
received a B.E. from the Tokyo University of Science, Japan, in 2003, and M.E. and Ph.D. in engineering from the Tokyo Institute of Technology, Japan, in 2005 and 2008. He is currently a Staff Researcher and HMI tech lead at the Toyota Research Institute in Los Altos, California. His research interests include HMIs for advanced safety systems, and augmented/mixed reality for vehicles.
\end{IEEEbiography}

\begin{IEEEbiography}[{\includegraphics[width=1in,height=1.25in,clip,keepaspectratio]{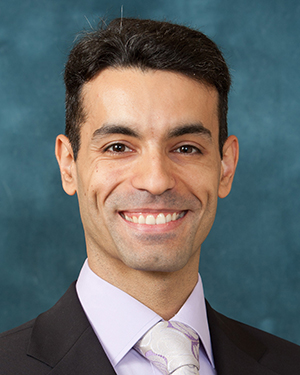}}]{Tulga Ersal}
received the B.S.E. degree from the Istanbul Technical University, Istanbul, Turkey, in 2001, and the M.S. and Ph.D. degrees from the University of Michigan, Ann Arbor, MI USA, in 2003 and 2007, respectively, all in mechanical engineering. He is currently an Associate Professor in the Department of Mechanical Engineering, University of Michigan, Ann Arbor. His research interests include modeling, simulation, and control of dynamic systems, with applications to vehicle and energy systems.
\end{IEEEbiography}

\vfill

\end{document}